\documentclass{article}

\usepackage{amsmath,amsfonts,stmaryrd,amsthm,amssymb,enumitem}
\usepackage{algorithm,booktabs,todonotes}

\usepackage{pifont}

\usepackage{tikz}
\usetikzlibrary{arrows}
\usetikzlibrary{calc}
\usetikzlibrary{shapes}
\usetikzlibrary{fit}
\usetikzlibrary{matrix} 
\usetikzlibrary{positioning}
\usetikzlibrary{decorations.pathreplacing}

\newtheorem{theorem}{Theorem}[section]
\newtheorem{lemma}[theorem]{Lemma}

\newtheorem{definition}[theorem]{Definition}

\renewcommand{\eqref}[1]{Eq.~(\ref{eq:#1})}
\newcommand{\figref}[1]{Figure~\ref{fig:#1}}
\newcommand{\tabref}[1]{Table~\ref{tab:#1}}        
\newcommand{\secref}[1]{Section~\ref{sec:#1}}
\newcommand{\thmref}[1]{Theorem~\ref{thm:#1}}
\newcommand{\lemref}[1]{Lemma~\ref{lem:#1}}

\newcommand{\appref}[1]{Appendix~\ref{app:#1}}

\newcommand{\myalgref}[1]{Alg.~\ref{alg:#1}}

\renewcommand{\P}{\mathbb{P}}
\newcommand{\E}{\mathbb{E}}

\newcommand{\reals}{\mathbb{R}}

\newcommand{\half}{{\frac12}}

\newcommand{\norm}[1]{\|#1\|}

\DeclareMathOperator*{\argmax}{argmax}

\newcommand{\st}{\text{ s.t. }}

\newcommand{\cL}{\mathcal{L}}

\newcommand{\cN}{\mathcal{N}}

\newcommand{\cX}{\mathcal{X}}

\usepackage{geometry}
\usepackage[utf8]{inputenc} 
\usepackage[T1]{fontenc}
\usepackage{lmodern}
\usepackage{hyperref}       
\usepackage{url}            
\usepackage{booktabs}       
\usepackage{amsfonts}       
\usepackage{nicefrac}       
\usepackage{microtype}      
\usepackage{tikz}
\usepackage{tikz-qtree}
\usepackage{bm}
\usepackage{natbib}
\usepackage{algorithmic}

\newcommand{\normone}[1]{\norm{#1}_1}
\newcommand{\bx}{\mathbf{x}}
\newcommand{\bw}{\mathbf{w}}
\newcommand{\bz}{\mathbf{z}}
\newcommand{\bZ}{\mathbf{Z}}
\newcommand{\bone}{\mathbf{1}}
\newcommand{\todolater}[1]{}
\newcommand{\spc}{\mathrm{SC}}
\newcommand{\algname}{\texttt{AWP}}
\newcommand{\pout}{P_o}
\newcommand{\dev}{\mathbb{D}}
\newcommand{\lfs}{\cL}

\newcommand{\papertitle}{Approximating a Target Distribution using Weight Queries}
\title{\papertitle}

\date{}

\author{
	Nadav Barak and Sivan Sabato \\
	Department of Computer Science\\
	Ben Gurion University\\
	Beer-Sheva, Israel \\
	\texttt{barakna@post.bgu.ac.il, sabatos@cs.bgu.ac.il }
}

\begin{document}
	\maketitle
\begin{abstract}
          We consider a novel challenge: approximating a distribution without the ability to randomly sample from that distribution. We study how such an approximation can be obtained using \emph{weight queries}. Given some data set of examples, a weight query presents one of the examples to an oracle, which returns the probability, according to the target distribution, of observing examples similar to the presented example. This oracle can represent, for instance, counting queries to a database of the target population, or an interface to a search engine which returns the number of results that match a given search. 
  
  We propose an interactive algorithm that iteratively selects data set examples and performs corresponding weight queries. The algorithm finds a reweighting of the data set that approximates the weights according to the target distribution, using a limited number of weight queries. We derive an approximation bound on the total variation distance between the reweighting found by the algorithm and the best achievable reweighting. Our algorithm takes inspiration from the UCB approach common in multi-armed bandits problems, and combines it with a new discrepancy estimator and a greedy iterative procedure. In addition to our theoretical guarantees, we demonstrate in experiments the advantages of the proposed algorithm over several baselines. A python implementation of the proposed
  algorithm and of all the experiments can be found at \url{https://github.com/Nadav-Barak/AWP}
\end{abstract}
\section{Introduction}\label{sec:intro}

A basic assumption
in learning and estimation tasks is the availability of a
random sample from the distribution of interest. However, in many cases, 
obtaining such a random sample is difficult or impossible. In this work, we
study a novel challenge: approximating a distribution without the ability to
randomly sample from it.  We consider a scenario in which the
only access to the distribution is via \emph{weight queries}. Given some data set of examples, a weight query presents one of these examples to an oracle, which returns the probability, according to the target distribution, of observing examples which are similar to the presented example. For instance, the available data set may list patients in a clinical trial, and the target distribution may represent the population of patients in a specific hospital, which is only accessible through certain database queries. In this case, the weight query for a specific data set example can be answered using a database counting query, which indicates how many records with the same demographic properties as the presented example exist in the database. A different example of a relevant oracle is that of a search engine for images or documents, which returns the number of objects in its database that are similar to the searched object. 

We study the possibility of using weight queries to find a reweighting of the input data set that approximates the target distribution. Importantly, we make no assumptions on the relationship between the data set and the target distribution. For instance, the data set could be sampled from a different distribution, or be collected via a non-random process. Reweighting the data set to match the true target weights would be easy if one simply queried the target weight of all the data set examples. In contrast, our goal in this work is to study whether a good approximated weighting can be found using a number of weight queries that is independent of the data set size. A data set reweighted to closely match a distribution is often used as a proxy to a random sample in learning and statistical estimation tasks, as done, for instance, in domain adaptation settings (e.g., \citealt{bickel2007discriminative,sugiyama2008direct,bickel2009discriminative}).

We consider a reweighting scheme in which the weights are fully defined by some  partition of the data set to a small number of subsets. Given the partition, the weights of examples in each subset of the partition are uniform, and are set so that the total weight of the subset is equal to its true target weight. 
An algorithm for finding a good reweighting should thus search for a partition such that within each of its subsets, the true example weights are as close to uniform as possible. For instance, in the hospital example, consider a case in which the hospital of interest has more older patients than their proportion in the clinical trial data set, but within each of the old and young populations, the makeup of the patients in the hospital is similar to that in the clinical trial. In this case, a partition based on patient ages will lead to an accurate reweighting of the data set. The goal of the algorithm is thus to find a partition which leads to an accurate reweighting, using a small number of weight queries. 

We show that identifying a good partition using only weight queries of individual examples is impossible, unless almost all examples are queried for their weight. We thus consider a setting in which the algorithm has a limited access to higher order queries, which return the total weight of a subset of the data set. These higher-order queries are limited to certain types of subsets, and the algorithm can only use a limited number of such queries. For instance, in the hospital example, higher-order queries may correspond to database counting queries for some  less specific demographic criteria. 

To represent the available higher order queries for a given problem, we assume that a hierarchical organization of the input data set is provided to the algorithm in the form of a tree, whose leaves are the data set examples. Each internal node in the tree represents the subset of leaves that are its descendants, and the only allowed higher-order queries are those that correspond to one of the internal nodes. 

Given such a tree, we consider only partitions that are represented by some
\emph{pruning} of the tree, where a pruning is a set of internal nodes such
that each leaf has exactly one ancestor in the set. A useful tree is one that
includes a small pruning that induces a partition of the data set into near-uniform subsets, as described above.

\textbf{Main results:} We give an algorithm that for any given tree, finds a near-optimal pruning of size $K$ using only $K-1$ higher order queries and $O(K^3/\epsilon^2)$ weight queries of individual examples, where $\epsilon$ measures the total variation distance of the obtained reweighting. The algorithm greedily splits internal nodes until reaching a pruning of the requested size. To decide which nodes to split, it iteratively makes weight queries of examples based on an Upper Confidence Bound (UCB) approach. Our UCB scheme employs a new estimator that we propose for the quality of an internal node. We show that this is necessary, since a naive estimator would require querying the weight of almost all data set examples. 
Our guarantees depend on a property of the input hierarchical tree that we term the \emph{split quality}, which essentially requires that local node splits of the input tree are not too harmful. We show that any algorithm that is based on iteratively splitting the tree and obtains a non-trivial approximation guarantee requires some assumption on the quality of the tree.

To supplement our theoretical analysis, we also implement the proposed algorithm and report several experiments, which demonstrate its advantage over several natural baselines. A python implementation of the proposed
algorithm and of all the experiments can be found at the following url: \url{https://github.com/Nadav-Barak/AWP}.

\subsection*{Related work}

In classical density estimation, the goal is to estimate the density function of a random variable given observed data \citep{silverman1986density}.
Commonly used methods are based on Parzen or Kernel estimators \citep{wand1994kernel,goldberger2005hierarchical}, expectation maximization (EM) algorithms \citep{mclachlan2007algorithm,figueiredo2002unsupervised} or variational estimation \citep{corduneanu2001variational,mcgrory2007variational}.
Some works have studied active variants of density estimation. In \citet{ghasemi2011active},  examples are selected for kernel density estimation. In \citet{kristan2010online}, density estimation in an online and interactive setting is studied. We are not aware of previous works that consider estimating a distribution using weight queries. 
The domain adaptation framework \citep{kifer2004detecting,ben2007analysis,blitzer2008learning,mansour2009domain,ben2010theory} assumes a target distribution with scarce or unlabeled random examples. In \citet{bickel2007discriminative,sugiyama2008direct,bickel2009discriminative}, reweighting the labeled source sample based on unlabeled target examples is studied. Trade-offs between source and target labels are studied in \citet{kpotufe18a}. 
In \citet{berlind2015active}, a labeled source sample guides target label requests. 

In this work, we assume that the input data set is organized in a hierarchical tree, which represents relevant structures in the data set. This type of input is common to many algorithms that require structure. For instance, such an input tree is used for active learning in \citet{dasgupta2008hierarchical}. In \citet{slivkins2011multi,bubeck2011x}, a hierarchical tree is used to organize different arms in a multi-armed-bandits problem, and in \citet{munos2011optimistic} such a structure is used to adaptively estimate the maximum of an unknown function. 
In \cite{pmlr-v119-cortes20a}, an iterative partition of the domain is used for active learning.

\section{Setting and Notations}
Denote by $\bone$ the all-1 vector; its size will always be clear from context. For an integer $n$, let $[n] := \{1,\ldots,n\}$. For a vector or a sequence $\bx= ( x(1),\ldots,x(n))$, let \mbox{$\normone{\bx} := \sum_{i \in [n]} |x(i)|$} be the $\ell_1$ norm of $\bx$. For a function $f:\cX \rightarrow \reals$ on a discrete domain $\cX$, denote \mbox{$\normone{f} := \sum_{x \in \cX} |f(x)|$}. \label{sec:settings} 

The input data set is some finite set $S \subseteq \cX$.
We assume that the examples in $S$ induce a partition on the domain $\cX$, where the part represented by $x \in S$ is the set of target examples that are similar to it (in some application-specific sense). The target weighting of $S$ is denoted $w^*:S\rightarrow [0,1]$, where $\normone{w^*} = 1$.  $w^*(x)$ is the probability mass, according to the unknown target distribution, of the examples in the domain that are represented by $x$.
For a set $V \subseteq S$, denote $w^*(V) := \sum_{x \in V}w^*(x)$.
The goal of the algorithm is to approximate the target weighting $w^*$ using weight queries, via a partition of $S$ into at most $K$ parts, where $K$ is provided as input to the algorithm.

A \emph{basic weight query} presents some $x \in S$ to the oracle and receives its weight $w^*(x)$ as an answer. To define the available higher-order queries, the input to the algorithm includes a binary hierarchical tree $T$ whose leaves are the elements in $S$. For an internal node $v$, denote by $\lfs_v \subseteq S$ the set of examples in the leaves descending from $v$.
A \emph{higher-order query} presents some internal node $v$ to the oracle, and receives its weight $w^*(\lfs_v)$ as an answer. We note that the algorithm that we propose below uses higher-order queries only for nodes of depth at most $K$. 

A reweighting algorithm attempts to approximate the target weighting by finding a small \emph{pruning} of the tree that induces a weighting on $S$ which is as similar as possible to $w^*$.
Formally, for a given tree, a pruning of the tree is a set of internal nodes such that each tree leaf is a descendant of exactly one of these nodes. Thus, a pruning of the input tree $T$ induces a partition on $S$. The weighting induced by the pruning is defined so that the weights assigned to all the leaves (examples) descending from the same pruning node $v$ are the same, and their total weight is equal to the true total weight $w^*(\lfs_v)$. 
Formally, let $N_v := |\lfs_v|$. For a pruning $P$ and an example $x \in S$, let $\textsf{A}(P,x)$ be the node in $P$ which is the ancestor of $x$.
The weighting $w_P:S \rightarrow \reals^+$ induced by the pruning $P$ is defined as $w_P(x) := w^*_{\textsf{A}(P,x)}/N_{\textsf{A}(P,x)}$.

The quality of a weighting $w_P$ is measured by the total variation distance between the distribution induced by $w_P$ and the one induced by $w^*$. This is equivalent to the $\ell_1$ norm between the weight functions (see, e.g., \citealp{wilmer2009markov}). Formally, the total variation distance between two weight functions $w_1,w_2:S\rightarrow \reals^+$ is
\begin{align*}
\mathrm{dist}(w_1,w_2) & :=  \max_{S' \subseteq S} |w_1(S') - w_2(S')| \\ & \equiv  \half\sum_{x \in S} |w_1(x) - w_2(x)| \equiv \half\normone{w_1 - w_2}.
\end{align*}

For a node $v$, define $\dev_v := \sum_{x \in \cL_v} |w^*_v/N_v - w^*(x)|.$ The \emph{discrepancy} of $w_P$ (with respect to $w^*$) is
\[
  \dev_P := 2\mathrm{dist}(w_P, w^*) = \normone{w_P - w^*} =  \sum_{v \in P}\dev_v.
\]
Intuitively, this measures the distance of the weights in each pruning node from uniform weights. 
  More generally, for any subset $G$ of a pruning, define $\dev_G := \sum_{v \in G} \dev_v$. We also call $\dev_G,\dev_v$ the discrepancy of $G$, $v$, respectively.
  The goal of the algorithm is thus to find a pruning $P$ with a low discrepancy $w_P$, using a small number of weight queries.

\section{Estimating the Discrepancy}\label{sec:estimator}
As stated above, the goal of the algorithm is to find a pruning with a low discrepancy. A necessary tool for such an algorithm is the ability to estimate the discrepancy $\dev_v$ of a given internal node using a small number of weight queries.
In this section, we discuss some challenges in estimating the discrepancy and present an estimator that overcomes them. 

First, it can be observed that the
discrepancy of a node cannot be reliably estimated from basic weight queries
alone, unless almost all leaf weights are queried. To see this, consider two cases: one in which all the leaves in $\lfs_v$ have an equally small weight, and one in which this holds for all but one leaf, which has a large weight. The discrepancy $\dev_v$ in the first case is zero, while it is large in the second case. However, it is impossible to distinguish between the cases using basic weight queries, unless they happen to include the heavy leaf. A detailed example of this issue is provided in \appref{estimator}.

To overcome this issue, the algorithm uses a higher order query to obtain the total weight of the internal node $w_v^*$, in addition to a random sample of basic weight queries of examples in $\lfs_v$. However, even then, a standard empirical estimator of the discrepancy, obtained by
aggregating \mbox{$|w^*_v/N_v - w^*(x)|$} over sampled examples, can have a 
large estimation error due to the wide range of possible values (see example in \appref{estimator}). We thus propose a different estimator, which circumvents this issue. 
The lemma below gives this estimator and proves a concentration bound for it. In this lemma, the leaf weights are represented by $\bw = (w_1,\ldots,w_n)$ and the discrepancy of the node is $\dev(\bw)$. In more general terms, this lemma gives an estimator for the uniformity of a set of values. 

\begin{lemma}\label{lem:bounds}
  Let $\bw = (w_1,\ldots,w_n)$ be a sequence of non-negative real values with a known $\normone{\bw}$. Let $W := \norm{\bw}/n$, and   
   $\dev(\bw) := \normone{\bw - W\cdot \mathbf{1}}$. Let $U$ be a uniform distribution over the indices in $[n]$, and suppose that $m$ i.i.d.~samples $\{I_i\}_{i \in [m]}$ are drawn from $U$. Denote $Z_i := w_{I_i}$ for $i \in [m]$, and $\bZ := (Z_1,\ldots,Z_m)$. Let the estimator for $\dev(\bw)$ be:
   \[
     \hat{\dev}(\bw) := \normone{\bw} +\frac{n}{m}( \normone{\bZ - W\cdot \bone} -
     \normone{\bZ}).
     \]
Then, with a  probability at least $1-\delta$, 
\begin{align*}
|\dev(\bw) - \hat{\dev}(\bw)| \le \normone{\bw}\sqrt{2\ln(2/\delta)/m} .
\end{align*}
\end{lemma}
\begin{proof}
  Let $Z_i' = |Z_i - W| - Z_i$. If $Z_i \geq W$, we have $Z_i' = -W$. Otherwise,  $0 \leq Z_i < W$, and therefore $Z_i' = W - 2Z_i \in [-W, W]$.

  Combining the two cases, we get $\P[Z_i' \in [-W,W]] = 1$.
Thus, by Hoeffding's inequality, with a probability at least $1-\delta$, we have
\begin{align*}
  &|\E[Z'_1] - \frac{1}{m}\sum_{i\in [m]}Z'_i| \leq W\cdot \sqrt{2\ln(2/\delta)/m}, \text{ and }\\
  &\E[Z'_1] = \frac{1}{n}(\normone{\bw - W\cdot \mathbf{1}} - \normone{\bw}) = \frac{1}{n}(\dev(\bw)-\normone{\bw}).
\end{align*}

      In addition,
      \[
        \frac{1}{m}\sum_{i\in [m]}Z'_i = \frac{1}{m}(\normone{\bZ - W\cdot \bone}-\normone{\bZ}) = \frac{1}{n}(\hat{\dev}(\bw)-\normone{\bw}).
        \]
        Therefore, with a probability at least $1-\delta$,
        \begin{align*}
      &\big|\dev(\bw) - \normone{\bw} - (\hat{\dev}(\bw)-\normone{\bw})\big|
      \leq
      nW\sqrt{2\ln(2/\delta)/m}  \\
      &\quad = \normone{\bw}\sqrt{2\ln(2/\delta)/m},
        \end{align*}
    as claimed.
    \end{proof}

\section{Main result: the \algname\ algorithm}\label{sec:main}

We propose the \algname\ (Approximated Weights via Pruning) algorithm, listed in \myalgref{alg}. \algname\ uses
weight queries to find a pruning $P$, which induces a weight function $w_P$ as defined in \secref{settings}. \algname\ gets the following inputs: a binary tree $T$ whose leaves are the data set elements, the requested pruning size $K \geq 2$, a confidence parameter $\delta \in (0,1)$, and a
constant $\beta > 1$, which controls the trade-off between the number of weight queries requested by \algname\ 
 and the approximation factor that it obtains.
Finding a pruning with a small discrepancy while limiting the number of weight queries involves
several challenges, since the discrepancy of any given pruning is unknown in advance, and the number of possibilities is exponential in $K$.
\algname\ starts with the trivial pruning, which includes only the root node. It iteratively samples weight queries of leaves to estimate the discrepancy of nodes in the current pruning. \algname\ decides in a greedy manner when to \emph{split} a node in the current pruning, that is, to replace it in the pruning with its two child nodes. It stops after reaching a pruning of size $K$.  

We first provide the notation for \myalgref{alg}.
For a node $v$ in $T$, $M_v$ is the number of weight queries of examples in $\lfs_v$ requested so far by the algorithm for estimating $\dev_v$.  
 The sequence of $M_v$ weights returned by the oracle for these queries is denoted $\bz_v$. Note that although an example in $\lfs_u$ is also in $\lfs_v$ for any ancestor $v$ of $u$, weight queries used for estimating $\dev_v$ are not reused for estimating $\dev_u$, since this would bias the estimate.  \algname\ uses the estimator for $\dev_v$ provided in \lemref{bounds}. In \algname\ notation, the estimator is: 
\begin{equation}\label{eq:hatdev}
  \hat{\dev}_v := w^*_v +\frac{N_v}{M_v}( \normone{\bz_v - \frac{w^*_v}{N_v}\cdot \bone} - \normone{\bz_v}).
\end{equation}
\algname\ iteratively samples weight queries of examples for nodes in the current pruning, until it can identify a node which has a relatively large discrepancy. The iterative sampling procedure takes
inspiration from the upper-confidence-bound (UCB) approach, common in
best-arm-identification problems \citep{audibert2010best}; In our case, the goal is to find the best node up to a multiplicative factor. 
In each iteration, the node with the
maximal known upper bound on its discrepancy is selected, and the weight of a random example from its leaves is queried. 
To calculate the upper bound, we define
\begin{equation}\label{eq:deltav}
  \Delta_v := w^*_v \cdot \sqrt{2 \ln(2K \pi^2 M_v^2/ (3\delta))/M_v}.
\end{equation}
We show in \secref{analysis} that $|\dev_v - \hat{\dev}_v|\leq \Delta_v$ with
a high probability. Hence, the upper bound for $\dev_v$ is set to $\hat{\dev}_v + \Delta_v$. 
Whenever a node from the pruning is identified as having a large discrepancy in comparison with the other nodes, it is replaced by its child nodes, thus increasing the pruning size by one. Formally, \algname\ splits a node if with a high probability,
\mbox{$\forall v' \in P\setminus \{v\}, \dev_v \geq \dev_{v'}/\beta$}. 
The factor of $\beta$ trades off the optimality of the selected node in terms of its discrepancy with the number of queries needed to identify such a node and perform a split. In addition, it makes sure that a split can be performed even if all nodes have a similar discrepancy. The formal splitting criterion is defined via the following Boolean function:
\begin{equation}\label{eq:split}
  \spc(v,P) \!=\! \left\{\beta( \hat{\dev}_v - \Delta_v) \!\ge\! \max_{v' \in P\setminus\{v\}}\hat{\dev}_{v'} + \Delta_{v'}\right\}.
\end{equation}

To summarize, \algname\ iteratively selects a node using the UCB criterion, and queries the weight of a random leaf of that node. Whenever the splitting criterion holds, \algname\ splits a node that satisfies it. This is repeated until reaching a pruning of size $K$. In addition, when a node is added to the pruning, its weight is queried for use in \eqref{hatdev}. 

We now provide our guarantees for \algname. The properties of the tree $T$ affect the quality of the output weighting. First, the pruning found by \algname\ cannot be better than the best pruning of size $K$ in $T$. Thus, we guarantee an approximation factor relative to that pruning.
In addition, we require the tree to be sufficiently nice, in that a child node should have a somewhat lower discrepancy than its parent. Formally, we define the notion of \emph{split quality}. 
\newcommand{\vsample}{{v_s}}
\newcommand{\mycomment}[1]{$\vartriangleright$ \emph{#1}}
\begin{algorithm}[t]
  \caption{\algname: Approximated weighting of a data set via a low-discrepancy pruning}
  \label{alg:alg}
\begin{algorithmic}[1]
  \STATE {\bfseries Input:}{A binary tree $T$ with examples as leaves; Maximum pruning size $K \geq 2$; $\delta \in (0,1)$; $\beta > 1$.}
  \STATE {\bfseries Output:}{$\pout$: A pruning of $T$ with a small $\dev_{\pout}$}
  \STATE Initializations: $P \leftarrow \{\text{The root node of }T\}$;
  \STATE For all nodes $v$ in $T$: $\bz_v \leftarrow (), M_v \leftarrow 0$.
  \WHILE{ $|P| < K$}
  \STATE \mycomment{Select a node to sample from}
\STATE  Set $\vsample \leftarrow \argmax_{v \in P} (\hat{\dev}_v + \Delta_v)$. 
\STATE Draw uniformly random example $x$ from $\lfs_{\vsample}$
\STATE Query the true weight of $x$,  $w^*(x)$.
\STATE $M_\vsample \leftarrow M_\vsample + 1$, $\bz_\vsample \leftarrow \bz_\vsample \circ w^*(x)$.
\STATE \mycomment{Decide whether to split a node in the pruning:}
\WHILE {$\exists v \in P \st \spc(v,P)$ holds (\eqref{split})}\label{line:sc}
\STATE Let $v_R,v_L$ be children of such a $v$.
\STATE Set $P \leftarrow P \setminus \{v\} \cup \{v_R, v_L\}$.\label{line:split}
\STATE Query $w^*_{v_R}$; set $w^*_{v_L}  \leftarrow w^*_{v} - w^*_{v_R}$. \label{line:nodeq}
\ENDWHILE
\ENDWHILE
\STATE return $\pout := P$.

    \end{algorithmic}
  \end{algorithm}

\begin{definition}[Split quality]\label{def:split}
  Let $T$ be a hierarchical tree for $S$, and $q \in (0,1)$. $T$ has a \emph{split quality} $q$ if for any two nodes $v,u$ in $T$ where $u$ is a child of $v$, we have $\dev_u \leq q \dev_v$. 
\end{definition}
This definition is similar in nature to other tree quality notions, such as the taxonomy quality of \citet{slivkins2011multi}, though the latter restricts weights and not the discrepancy.  We note that Def.~\ref{def:split} could be relaxed, for instance by allowing different values of $q$ in different tree levels. Nonetheless, we prove in \appref{split} that greedily splitting the node with the largest discrepancy cannot achieve a reasonable approximation factor without some restriction on the input tree, and that this also holds for other types of greedy algorithms. 
It is an open problem whether this limitation applies to all greedy algorithms. Note also that even in trees with a split quality less than $1$, splitting a node might increase the total discrepancy; see \secref{analysis} for further discussion. Our main result is the following theorem.

\begin{theorem}\label{thm:main}
  Suppose that \algname\ gets the inputs $T$, $K \geq 2$, $\delta \in (0,1)$, $\beta > 1$ and let $\pout$ be its output. Let $Q$ be a pruning of $T$ of size $K$ with a minimal $\dev_Q$. With a probability at least $1-\delta$, we have:
  \begin{itemize}
  \item If $T$ has split quality $q$ for some $q < 1$, then  \[\dev_{\pout} \le 2\beta (\frac{\log(K)}{\log(1/q)}+1) \dev_Q.\]
    \item \algname\ requests $K-1$ weight queries of internal nodes.
  \item \algname\ requests $\tilde{O}( (1+\frac{1}{\beta-1})^2K^3\ln(1/\delta)/\dev_{\pout}^2)$ weight queries of examples.\footnote{The $\tilde{O}$ notation hides constants and logarithmic factors; These are explicit  in the proof of the theorem.}
  \end{itemize}
\end{theorem}
Thus, an approximation factor with respect to the best achievable discrepancy is obtained, while keeping the number of higher-order weight queries minimal and bounding the number of weight queries of examples requested by the algorithm. The theorem is proved in the next section.

\section{Analysis}\label{sec:analysis} 
In this section, we prove the main result, \thmref{main}.
First, we prove the correctness of the definition of $\Delta_v$ given in \secref{main}, using the concentration bound given in \lemref{bounds} for the estimator $\hat{\dev}_v$. The proof is provided in \appref{prob}. 
\begin{lemma}\label{lem:prob}
  Fix inputs $T,K,\delta,\beta$ to $\algname$. Recall that $P$ is the pruning updated by \algname\ during its run. The following event holds with a probability at least $1 - \delta$ on the randomness of \algname:
  \begin{equation}\label{eq:e0}
  \begin{split}
  E_0 := \{ & \text{At all times during the run of \algname,} \\
  	 & \forall v \in P, |\dev_v - \hat{\dev}_v| \le \Delta_v \}.
  \end{split}
  \end{equation}
\end{lemma}

Next, we bound the increase in discrepancy that could be caused by a node split. Even in trees with a split quality less than $1$, a split could increase the discrepancy of the pruning. The next lemma bounds this increase, and shows that this bound is tight. The proof is provided in \appref{clowerbound}. 
\begin{lemma}\label{lem:sons}
  Let $r$ be the root of a hierarchical tree and let $P$ be a pruning of this tree. Then $\dev_P \leq 2\dev_r$. Moreover, for any $\epsilon > 0$, there exists a tree with a split quality $q < 1$ and a pruning $P$ of size $2$ such that $\dev_P \geq (2-\epsilon)\dev_r$.
\end{lemma}

We now prove the two main parts of \thmref{main}, starting with the approximation factor of \algname. In the proof of the following lemma, the proof of some claims is omitted. The full proof is provided in \appref{prooflem44}. 
 \begin{lemma}\label{lem:devbound}
   Fix inputs $T,K,\delta,\beta$ to $\algname$, and suppose that $T$ has a split quality $q\in (0,1)$. Let $\pout$ be the output of $\algname$. Let $E_0$ be the event defined in \eqref{e0}. In any run of \algname\ in which $E_0$ holds, for any pruning $Q$ of $T$ such that $|Q| = K$, we have $\dev_{\pout} \le 2\beta (\frac{\log(K)}{\log(1/q)}+1) \dev_Q$. 
       \end{lemma}

\begin{proof}
	
  Let $Q$ be some pruning such that $|Q| = K$.  Partition $\pout$ into $R,P_a$ and $P_d$, where $R := \pout \cap Q$, $P_a \subseteq \pout$ is the set of strict ancestors of nodes in $Q$, and $P_d \subseteq \pout$ is the set of strict descendants of nodes in $Q$.  Let $Q_a \subseteq Q$  be the ancestors of the nodes in $P_d$ and let $Q_d\subseteq Q$ be the descendants of the nodes in $P_a$, so that $R,Q_d$ and $Q_a$ form a partition of $Q$. First, we prove that we may assume without loss of generality that $P_a,P_d,Q_a,Q_d$ sets are non-empty. 

  \textbf{Claim 1}: 
  If any of the sets $P_a,P_d,Q_a,Q_d$ is empty then the statement of the lemma holds.

  The proof of Claim 1 is deferred to \appref{prooflem44}.
  	 Assume henceforth that $P_a,P_d,Q_a,Q_d$ are non-empty.  Let $r$ be the node with the smallest discrepancy out of the nodes that were split by \algname\ during the entire run. Define $\theta := |P_a|\cdot \dev_r/\dev_{Q_a}$ if $\dev_{Q_a}>0$  and $\theta := 0$ otherwise.

  \textbf{Claim 2}: $\dev_{\pout} \leq \max(2,\beta \theta) \dev_Q.$

  \textbf{Proof of Claim 2}: 
	 We bound the discrepancies of $P_d$ and of $P_a$ separately.
	 For each node $u\in Q_a$, denote by $P(u)$ the descendants of $u$ in $P_d$. These form a pruning of the sub-tree rooted at $u$. In addition, the sets $\{P(u)\}_{u \in Q_a}$ form a partition of $P_d$. Thus, by the definition of discrepancy and \lemref{sons}, 
	 \begin{equation}\label{eq:pd}
		\dev_{P_d} = \sum_{u \in Q_a} \dev_{P(u)} \le \sum_{u \in Q_a} 2 \dev_u = 2 \dev_{Q_a}.
              \end{equation}
              Let $P$  be the pruning when \algname\ decided to split node $r$. By the definition of the splitting criterion $\spc$ (\eqref{split}), for all $v \in P\setminus \{r\}$, at that time it held that $\beta(\hat{\dev}_r - \Delta_r) \ge \hat{\dev}_{v} + \Delta_{v}$. Since $E_0$ holds, we have $\dev_r \geq \hat{\dev}_r - \Delta_r$ and $\hat{\dev}_{v} + \Delta_{v} \geq \dev_v$. Therefore,  $\forall v \in P\setminus \{r\}, \beta \dev_r \geq \dev_v$.

              Now, any node $v'\in \pout \setminus P$ is a descendant of some node $v \in P$. Since $T$ has split quality $q$ for $q<1$, we have $\dev_{v'} \leq \dev_v$. Therefore, for all $v' \in \pout$, $\dev_{v'} \leq \beta \dev_r$. In particular, $\dev_{P_a} \equiv \sum_{v\in P_a} \dev_v \leq \beta |P_a|\dev_r$. Since all nodes in $Q_a$ were split by \algname\ $\dev_{Q_a} = 0$ implies $\dev_r = 0$, therefore in all cases $\dev_{P_a} \leq \beta \theta \dev_{Q_a}$. Combining this with \eqref{pd}, we get that
    \begin{equation*}
    \begin{split}
    \dev_{\pout} & = \dev_{R} +\dev_{P_d}  + \dev_{P_a} \le \dev_R +  2 \dev_{Q_a}+ \beta \theta \dev_{Q_a} \\ & \leq \max(2,\beta \theta) \dev_Q,
    \end{split}
\end{equation*}
which completes the proof of Claim 2.

It follows from Claim 2 that to bound the approximation factor, it suffices to bound $\theta$. Let $P_d'$ be the set of nodes both of whose child nodes are in $P_d$ and denote
$n:=|P_d'|$. In addition, define
\[
  \alpha := \frac{\log(1/q)}{\log(|P_a|)+\log(1/q)} \leq 1.
\]
We now prove that $\theta \leq 2/\alpha$ by considering two complementary cases, $n \geq \alpha|P_a|$ and $n < \alpha|P_a|$. The following claim handles the first case. 

\textbf{Claim 3:} if $n \geq \alpha|P_a|$, then $\theta \leq 2/\alpha$.

\textbf{Proof of Claim 3}: 
Each node in $P_d'$
has an ancestor in $Q_a$, and no ancestor in $P_d'$. Therefore, $P_d'$ can be partitioned to subsets
according to their ancestor in $Q_a$, and each such subset is a part of some
pruning of that ancestor. Thus, by \lemref{sons}, 
$\dev_{P_d'} \leq 2 \dev_{Q_a}$. Hence, for some node $v \in P'_d$, $\dev_v \leq 2 \dev_{Q_a}/n$. It follows from the definition of $r$ that
$\dev_r \leq 2 \dev_{Q_a}/n$. Hence, $\theta \leq 2|P_a|/n$. Since $n \geq \alpha|P_a|$, we have $\theta \leq 2/\alpha$ as claimed.

We now prove this bound hold for the case $n < \alpha |P_a|$. For a node $v$ with an ancestor in $Q_a$, let $l_v$ be the path length from this ancestor to $v$, and define $L := \sum_{v \in P'_d} l_v$. We start with an auxiliary Claim 4.

\textbf{Claim 4}: $L \geq |P_a| - n$.

The proof of Claim 4 is deferred to \appref{prooflem44}. We use Claim 4 to prove the required upper bound on $\theta$ in Claim 5. 

        \textbf{Claim 5}: if $n < \alpha|P_a|$, then $\theta \leq 2/\alpha$.

\textbf{Proof of Claim 5}: 
It follows from Claim 4 that for some node $v \in P_d'$, $l_v \geq (|P_a| - n)/n = |P_a|/n-1 > 0$, where the last inequality follows since $n < \alpha |P_a| < |P_a|$. Letting $u \in Q_a$ be the ancestor of $v$ in $Q_a$, we have by the split quality  $q$ of $T$ that $\dev_v \leq \dev_u\cdot q^{\frac{|P_a|}{n}-1}$. Since $u \in Q_a$, we have $\dev_u \leq \dev_{Q_a}$. In addition, $\dev_r \leq \dev_v$ by the definition of $r$. Therefore, $\dev_r \leq \dev_{Q_a}\cdot q^{\frac{|P_a|}{n}-1}$. 
Since $n < |P_a|\alpha$ and $q<1$, from the definition of $\alpha,\theta$ we have

  $\theta \le |P_a|q^{\frac{|P_a|}{n}-1}\leq 1 \leq 2/\alpha.$

This proves Claim 5.

Claims 3 and 5 imply that in all cases, $\theta \leq 2/\alpha$. By Substituting $\alpha$, we have that
\[
  \theta \leq 2 (\frac{\log(|P_a|)}{\log(1/q)}+1) \leq 2 (\frac{\log(K)}{\log(1/q)}+1).
  \]
Placing this upper bound in the statement of Claim 2 concludes of the lemma. 
\end{proof}

Next, we prove an upper bound on the number of weight queries requested by \algname. 

\newcommand{\dmax}{\dev_{\mathrm{max}}}
\begin{lemma}\label{lem:queries}
  Let $\beta > 1$. Consider a run of \algname\ in which $E_0$ holds, and fix some iteration. Let $P$ be the current pruning, and for a node $v$ in $T$, denote \[
    \alpha_v := \frac{\beta}{\beta-1}\cdot \frac{w_v^*}{\max_{u\in P} \dev_{u}}.
  \] Then in this iteration, the node $\vsample$ selected by \algname\ satisfies $M_\vsample < 22\alpha^2_\vsample(\ln(\alpha^4_\vsample K/\delta)+10)$.

\end{lemma}
\begin{proof}
  Recall that the next weight query to be sampled by \algname\ is set to $\vsample \in \argmax_{v\in P} (\hat{\dev}_{v} + \Delta_{v})$. First, if some nodes $v \in P$ have $M_v = 0$, they will have $\Delta_v = \infty$, thus one of them will be set as $\vsample$, in which case the bound on $M_\vsample$ trivially holds. Hence, we assume below that for all $v \in P$, $M_v \geq 1$. Denote $\dmax := \max_{v\in P} \dev_{v}$.
   Let $u \in \argmax_{v\in P} \hat{\dev}_{v}$ be a node with a maximal estimated discrepancy.
	Since $E_0$ holds, for all $v \in P$ we have $\hat{\dev}_{v} + \Delta_{v} \ge \dev_v$. Thus, by definition of $\vsample$, $\hat{\dev}_{\vsample} + \Delta_{\vsample} \ge \dmax$ which implies $\hat{\dev}_u \geq \hat{\dev}_{\vsample} \geq \dmax - \Delta_{\vsample}$. Therefore, 
	\begin{align*}
	\beta (\hat{\dev}_{u} - \Delta_{u})  & = \hat{\dev}_{u} + (\beta -1)\hat{\dev}_{u} - \beta \Delta_{u} \\ & \ge \hat{\dev}_{u} + (\beta -1)(\dmax - \Delta_{\vsample}) -\beta \Delta_{u}.
	\end{align*}
	Denote $\theta := (\beta -1)\dmax/(2\beta)$, 
	and assume for contradiction that $\Delta_{\vsample} \le \theta$. By the definitions of $u$ and $\vsample$, we have $\Delta_{u} \leq \Delta_{\vsample} \leq \theta$. Thus, from the inequality above and the definitions of $\theta, u, \vsample$, 
	\begin{align}\label{eq:nodesplit}
          &\beta (\hat{\dev}_{u} - \Delta_{u})   \ge \hat{\dev}_{u} + (\beta - 1)(\dmax - \theta) - \beta\theta  \\
                                               &\quad =\hat{\dev}_{u} + \theta \ge \hat{\dev}_{\vsample} + \Delta_{\vsample} \geq \max_{z \in P\setminus \{u\}} (\hat{\dev}_{z} + \Delta_{z}), \notag
	\end{align}
        implying that $\spc(u,P)$ holds. But this means that the previous iteration should have split this node in the pruning, a contradiction.
        Therefore, $\Delta_{\vsample} > \theta$.
        Since by \eqref{deltav}, $\Delta_\vsample \equiv w^*_\vsample \sqrt{2 \ln( 2K \pi^2 M_\vsample^2 / (3 \delta))/M_\vsample}$, it follows that
        \[
          M_\vsample < \frac{2{w_\vsample^*}^2}{\theta^2}\cdot \ln(\frac{2K \pi^2 M_\vsample^2}{3 \delta}).
        \]
        Denoting $\phi = {4w_\vsample^*}^2/\theta^2$, $\mu = \sqrt{2K \pi^2 / (3 \delta)}$, this is equivalent to $M_\vsample < \phi \ln(\mu M_\vsample)$.  By \lemref{mathineq} in \appref{mathineq}, this implies $M_\vsample < e\phi \ln(e\mu\phi)$. Noting that $\phi = 16\alpha_v^2$ and bounding constants, this gives the required bound.
    \end{proof}
        
               Lastly, we combine the lemmas to prove the main theorem.
               
\begin{proof}[Proof of \thmref{main}]
  Consider a run in which $E_0$ holds. By \lemref{prob}, this occurs with a probability of $1-\delta$. If $T$ has a split quality $q < 1$, the approximation factor follows from  \lemref{devbound}.
  Next, note that \algname\ makes a higher-order query of a node only if it is the right child of a node that was split in line \ref{line:split} of \myalgref{alg}. Thus, it makes $K-1$ higher-order weight queries. This proves the first two claims.
  
  We now use \lemref{queries} to bound the total number of weight queries of examples under $E_0$. First, we lower bound $\max_{v\in P} \dev_v$ for any pruning $P$ during the run of \algname. Let $u \in \argmax_{v \in \pout} \dev_v$. By the definition of $u$, we have $\dev_u \ge \dev_{\pout} / K$. Now, at any iteration during the run of \algname, some ancestor $v$ of $u$ is in $P$. By \lemref{sons}, we have $\dev_v \geq \dev_u  / 2$. Therefore, $\max_{v\in P} \dev_v \geq \dev_{\pout}/(2K)$. 
  Thus, by \lemref{queries}, at any iteration of \algname, $\vsample$ is set to some node in $P$ that satisfies \mbox{$M_\vsample < 22\alpha^2_\vsample(\ln(\alpha^4_\vsample K/\delta)+10)$}, where \mbox{$\alpha_\vsample := \frac{\beta}{\beta-1}\cdot \frac{2K w_v^*}{\dev_{\pout}}$}. Hence, \algname\ takes at most \mbox{$T_v :=22\alpha^2_\vsample(\ln(\alpha^4_\vsample K/\delta)+10)+1$} samples for node $v$, and so the  the total number of example weight queries taken by \algname\ is at most $\sum_{v \in V} T_v$, where $V$ is the set of nodes that participate in $P$ at any time during the run. To bound this sum, note that for any pruning $P$ during the run of \algname, we have $\sum_{v\in P}w_v^* = 1$. Hence, $\sum_{v\in P}{w_v^*}^2 \leq 1$. Since there are $K$ different prunings during the run, we have $\sum_{v\in V}{w_v^*}^2 \leq K$. Substituting $\alpha_v$ by its definition and rearranging, we get that the total number of example weight queries by \algname\ is $\tilde{O}( (1+\frac{1}{\beta-1})^2K^3\ln(1/\delta)/\dev_{\pout}^2)$.
 \end{proof}

\section{Experiments}\label{sec:exps}
\newcommand{\weightalg}{\texttt{WEIGHT}}
\newcommand{\staticalg}{\texttt{UNIFORM}}
\newcommand{\empalg}{\texttt{EMPIRICAL}}

We report experiments that compare \algname\ to several natural baselines. The full results of the experiments described below, as well as results for additional experiments, are reported in  \appref{exps}. A python implementation of the proposed
algorithm and all experiments can be found at \url{https://github.com/Nadav-Barak/AWP}

The implementation of \algname\ includes two practical improvements: First, we use an empirical Bernstein concentration bound \citep{MaurerPo09} to reduce the size of $\Delta_v$ when possible; This does not affect the correctness of the analysis. See \appref{bernstein} for details.
Second, for all algorithms, we take into account the known weight values of single examples in the output weighting, as follows. For $v \in \pout$, let $S_v$ be the examples in $\lfs_v$ whose weight was queried. Given the output pruning $P$, we define the weighting $w_P'$ as follows. For $x \in S_v$, $w_P'(x) := w^*(x)$; for $x \in \lfs_v \setminus S_v$, we set $w'_P(x) := (w^*_v - \sum_{x \in S_v}w^*(x))/(N_v - |S_v|)$. In all the plots, we report the normalized output distance $\mathrm{dist}(w'_P,w^*) = \normone{w'_P - w^*}/2 \in [0,1]$, which is equal to half the discrepancy of $w_P'$.  
\begin{figure}[h]
  \begin{center}
    \includegraphics[width=0.45\textwidth]{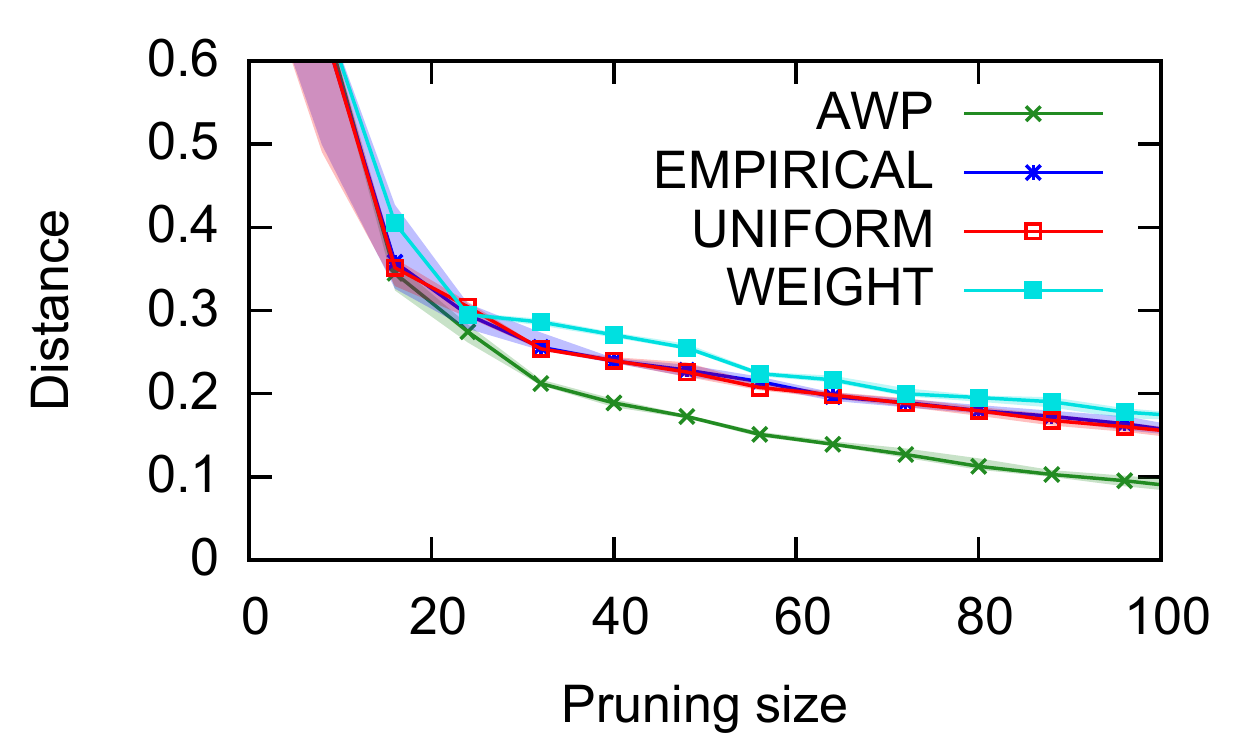}
    \includegraphics[width=0.45\textwidth]{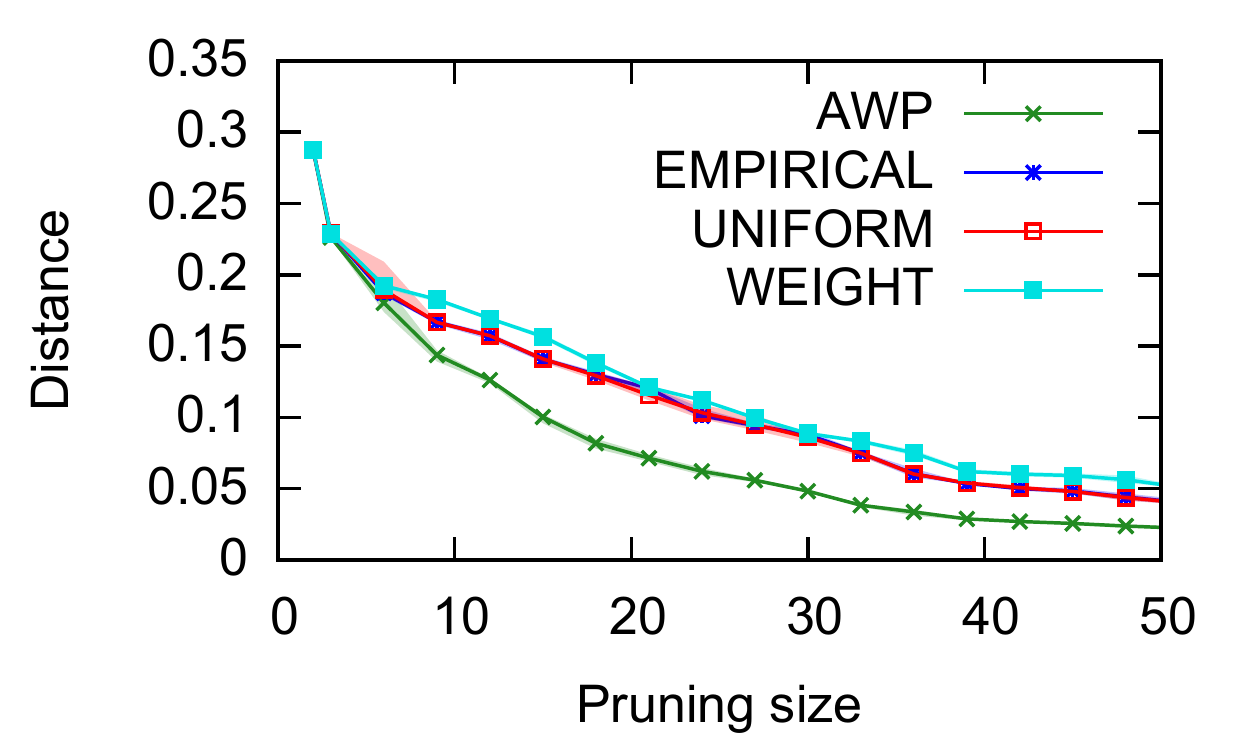}
    \\
    \includegraphics[width=0.45\textwidth]{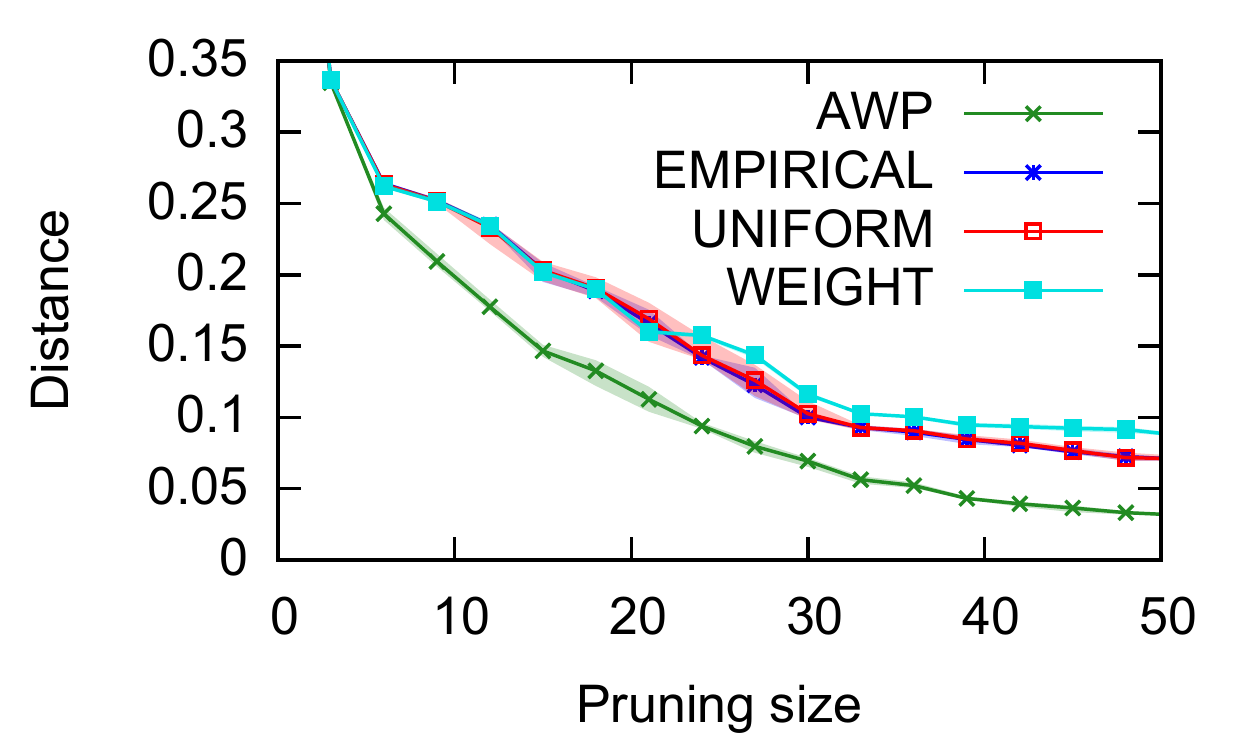} 
	\includegraphics[width=0.45\textwidth]{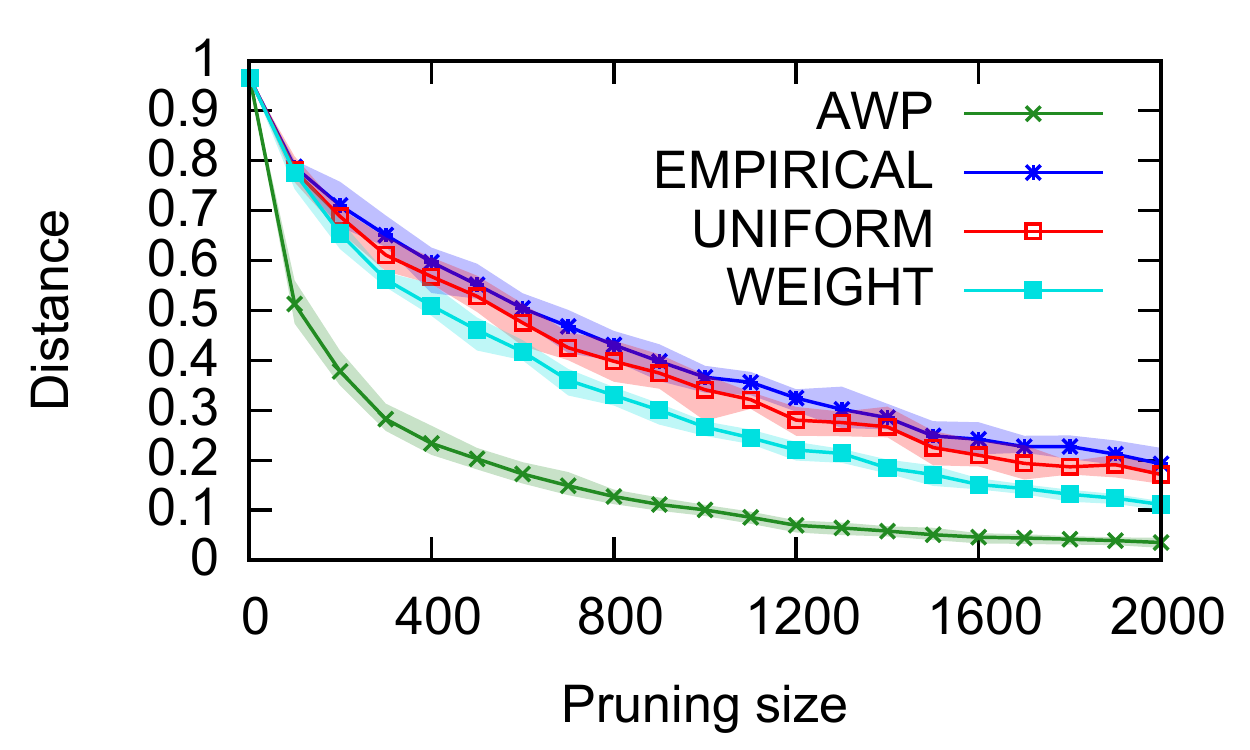}

      \end{center}
      \caption{Some of the experiment results. Full results are provided in \appref{exps}. Top to Bottom: Adult data set, bins assigned by marital status, \mbox{$N=4$}; MNIST: \mbox{$N=2$}, \mbox{$N=4$}; The  Caltech$\rightarrow$Office experiment.}
      \label{fig:exps}
    \end{figure}

To fairly compare \algname\ to the baselines, they were allowed the same number of higher-order weight queries and basic weight queries as requested by \algname\ for the same inputs. The baselines are non-adaptive, thus the basic weight queries were drawn uniformly from the data set at the start of their run. 
We tested the following baselines:
(1)~\weightalg: Iteratively split the node with the largest weight. Use basic weight queries only for the final calculation of $w_P'$. 
(2) \staticalg: Iteratively split the node $v$ with the largest $\hat{\dev}_v$, the estimator used by \algname, but without adaptive queries.
(3) \empalg: Same as \staticalg, except that instead of $\hat{\dev}_v$, it uses the naive empirical estimator, $\frac{n}{|S_v|}\sum_{x \in S_v}|w^*_v(x)/N_v-w^*(x)|$ (See \secref{estimator}).

We ran several types of experiments, all with inputs $\delta = 0.05$ and
$\beta = 4$. Note that $\beta$ defines a trade-off between the number of queries and the quality of the solution. Therefore, its value represents user preferences and not a hyper-parameter to be optimized. For each experiment, we report the average normalized output distance over $10$ runs, as a function of the pruning size. Error bars, displayed as shaded regions,
show the maximal and minimal normalized distances obtained in these runs.
For each input data set, we define an input hierarchical tree $T$ and test several target distributions, which determine the true weight function $w^*$.

In the first set of experiments, the input data set was Adult \citep{Dua:2019}, which contains $\sim 45,\!000$ population census records (after excluding those with missing values). We created the hierarchical tree via a top-down procedure, in which each node was split by one of the data set attributes, using a balanced splitting criterion. 
Similarly to the hospital example given in \secref{intro}, the higher-order queries (internal nodes) in this tree require finding the proportion of the population with a specific set of demographic characteristics, which could be obtained via a database counting query. For instance, a higher-order query could correspond to the set of single government employees aged 45 or higher. 
To generate various target distributions, we partitioned the data set into ordered bins, where in each experiment the partition was based on the value of a different attribute. The target weights were set so that each example in a given bin was $N$ times heavier than each example in the next bin. We ran this experiment with values $N=2,4$. Some plots are given in \figref{exps}. See \appref{exps} for full details and results.

The rest of the experiments were done on visual data sets, with a tree that was generated synthetically from the input data set using Ward's method \citep{mullner2011modern}, as implemented in the
\texttt{scipy} python package. 

First, we tested the MNIST \citep{MNIST} training set, which contains $60,\!000$ images of hand-written digits, classified into $10$ classes (digits). The target distributions were generated by allocating the examples into ordered bins based on the class labels, and allocating the weights as in the Adult data set experiment above, again with $N = 2,4$. We tested three random bin orders.
See \figref{exps} for some of the results. See \appref{exps} for full details and results.

Lastly, we ran experiments using an input data set and target distribution based on data set pairs commonly studied in domain adaptation settings
(e.g., \citealt{gong2012geodesic,hoffman2012discovering,ding2015deep}). The input data set was Caltech256 \citep{griffin2007caltech}, with $29,\!780$ images of various objects, classified into 256 classes (not including the singular ``clutter'' class).
In each experiment, the target distribution was determined by a different data set as follows: The target weight $w^*$ of each image in the Caltech256 data set was set to the fraction of images from the target data set that have this image as their nearest neighbor. We used two target data sets: (1) The Office data set \citep{saenko2010adapting}, of which we used the $10$ classes that also exist in Caltech256 ($1410$ images); (2) The Bing data set \citep{BergamoTorresani10,Bing}, which includes $300$ images in each Caltech256 class. For Bing, we also
ran experiments where images from a single super-class from the taxonomy in \citet{griffin2007caltech} were used as the target data set.
See \figref{exps} for some of the results, and \appref{exps} for full results.

In all experiments, except for a single configuration, \algname\ performed better than the
other algorithms. 
In addition, \staticalg\ and \empalg\ behave similarly in
most experiments, with  \staticalg\ sometimes being slightly better. This shows that our new estimator is empirically adequate, on top of its crucial advantage in getting a small $\Delta_v$. We note that in our experiments, the split quality $q$ was usually close to one. This shows that \algname\ can be successful even in cases not covered by \thmref{main}. 
We did find that the average splitting values were usually lower, see \appref{qval}.

\section{Conclusions}

In this work, we studied a novel problem of approximating a distribution via weight queries, using a pruning of a hierarchical tree. We showed, both theoretically and experimentally, that such an approximation can be obtained using an efficient interactive algorithm which iteratively constructs a pruning. In future work, we plan to study the effectiveness of our algorithm under more relaxed assumptions, and to generalize the input structure beyond a hierarchical tree.

\clearpage

\bibliographystyle{plainnat}

\bibliography{discrepancy_estimation}

\clearpage
\onecolumn

\appendix

\section{Limitations of greedy algorithms}\label{app:split}
In this section we prove two lemmas which point to limitations of certain types of greedy algorithms for finding a pruning with a low discrepancy. 
The first lemma shows that without a restriction on the split quality of the input tree, the greedy algorithm which splits the node with the maximal discrepancy, as well as a general class of greedy algorithms, could obtain poor approximation factors. 
\begin{lemma}\label{lem:approxlower}
  Consider a greedy algorithm which creates a pruning by starting with the singleton pruning that includes the root node, and iteratively splitting the node with the largest discrepancy in the current pruning. Then, for any even pruning size $K \geq 2$, there exists an input tree such that the approximation factor of the greedy algorithm is at least $K/4$.
  
  Moreover, the same holds for any greedy algorithm which selects the next node to split based only on the discrepancy of each node in the current pruning and breaks ties arbitrarily.

  In both cases, the input tree that obtains this approximation factor does not have a split quality $q < 1$, but does satisfy the following property (equivalent to having a split quality of $q=1$): For any two nodes $v,u$ in $T$ such that $u$ is a child of $v$, $\dev_u \leq \dev_v$. 
\end{lemma}

\begin{figure}[h]
  \begin{center}
    \begin{tikzpicture}[scale = 0.8]
      \Tree [.{$T_1\, (\dev =w)$} {$0$} {$w$} ];
    \end{tikzpicture}
\begin{tikzpicture}[scale = 0.8]
  \Tree [.{$T_2\,(\dev = 2w)$} {$T_1$} {$T_1$} ]
\end{tikzpicture}
\begin{tikzpicture}[scale = 0.8]
  \Tree [.{$T_3\,(\dev = 3w)$} {$T_1$} [.{$T_2$} {$T_1$} {$T_1$} ] ]
\end{tikzpicture}\\

\vspace{2em}
    \begin{tikzpicture}[scale = 0.8]
\Tree [.{$G_j(4)\,(\dev = jw)$} {$w/2$ (root: $u_4$)} [.{$G_j(3)\,(\dev = jw)$} {$w/2$ (root: $u_3$)} [.{$G_j(2)\,(\dev = jw)$} {$w/2$ (root: $u_2$)} {$G_j(1) = T_{j}\,(\dev = jw)$ (root: $u_1$)} ] ] ]
\end{tikzpicture}

\vspace{2em}

\begin{tikzpicture}[scale = 0.8]
      \Tree [.{$T^a = H(3)\,(\dev = 4w)$} {$T_1$ (root: $v_3$)} [.{$H(2)\,(\dev = 3w)$} {$T_1$ (root: $v_2$)} {$H(1)\equiv G_2(4)\,(\dev = 2w)$} ] ] ]
\end{tikzpicture}
\end{center}

\caption{Illustrating the trees defined in the proof of \lemref{approxlower} for $k = 3$.}
\label{fig:trees}
\end{figure}
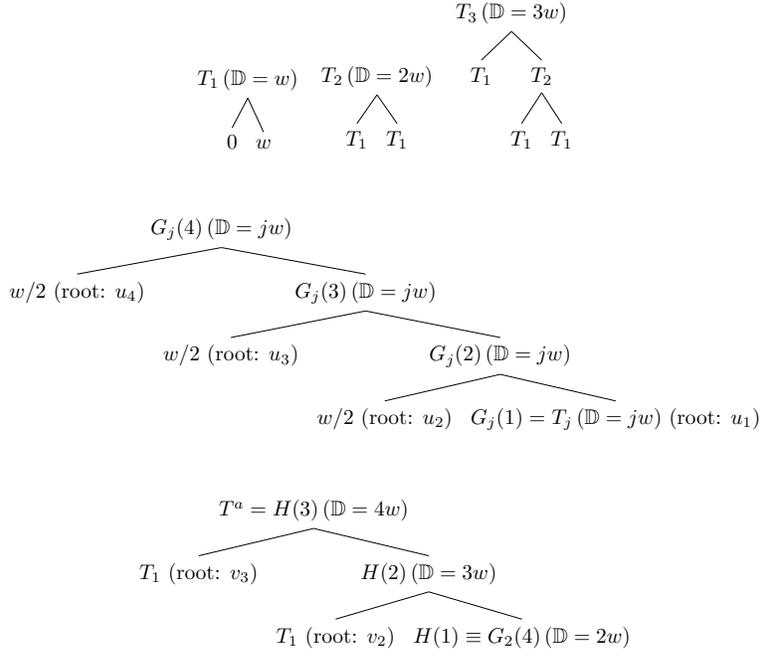

\begin{proof}
	  We define several trees; see illustrations in \figref{trees}. Let $w >
  0$. Its value will be defined below.  All the trees defined below have an average leaf weight of $w/2$. Therefore, when recursively combining them to a larger tree, the average weight remains the same, and so the discrepancy of any internal node (except for nodes with leaf children) is the total discrepancy of its two child nodes.
  
  Let $T_1$ be a tree of depth $1$, which has a root node with two child leaves with weights $0$ and $w$. The root of $T_1$ has a discrepancy of $w$. For $i \geq 2$, let $T_i$ be a tree of depth $i$ such that one child node of the root is $T_1$ and the other is $T_{i-1}$. Note that $T_i$ has a discrepancy of $iw$. 

  .
  For positive integers $i$ and $j$, define the tree $G_j(i)$ recursively, such that $G_j(1) := T_j$ (denote its root node $u_1$), and $G_j(i)$ has a root node with two children: a leaf with weight $w/2$ (denote it $u_i$) and $G_j(i-1)$.  It is easy to verify that the
  discrepancy of $G_j(i)$ for all $i \geq 1$ is $jw$.

Let $k = K/2$. For the first part of the lemma, define $H(i)$ recursively. $H(1) = G_2(k+1)$,  and $H(i)$ has a root with the children $T_1$ (denote its root $v_i$) and $H(i-1)$.  The discrepancy of $H(i)$ is thus $(i+1)w$. 
  Now, consider the greedy algorithm that iteratively splits the node with the largest discrepancy in the pruning, and suppose that it is run with the input tree $T^a := H(k)$ and a pruning size $K = 2k$. Set $w$ so that the total weight of $T^a$ is equal to $1$. 
  Due to the discrepancy values, the greedy algorithm splits the root nodes of $H(k), H(k-1),\ldots,H(1)=G_2(k+1)$ and then of $G_2(k+1), G_2(k),\ldots,G_2(2)$. The resulting pruning is $v_2,\ldots,v_k,u_1,\ldots,u_{k+1}$, with a total discrepancy of $(k-1)w + 2w = (k+1)w$.
  In contrast, consider the pruning of size $K$ which includes the two leaf children of each of $v_2,\ldots,v_k$ and the children of the root of $G_2(k+1)$ (the sibling of $v_2$). This pruning has a discrepancy of $2w$. Thus, the approximation factor obtained by the greedy algorithm in this example is $(k+1)/2 \geq K/4$.

  For the second part of the lemma, consider the tree $T^b$ which has a root with the child nodes $G_1(2k)$ (which has a discrepancy of $w$) and $T_{k-1}$ (which has a discrepancy of $(k-1)w$). Define $w$ so that the total weight of $T^b$ is equal to $1$. Suppose that $T^b$ and pruning size $K$ are provided as input to some greedy algorithm that splits according to discrepancy values of nodes in the current pruning, and breaks ties arbitrarily. In the first round, the root node must be split. Thereafter, the current pruning always includes some pruning of $G_1(2k)$ (possibly the singlton pruning which includes just the root of this sub-tree). There is only one pruning of $G_1(2k)$ of size $i \leq 2k$, and it is composed of $i-1$ leaves of weight $w/2$ and the root of $G_1(2k-i+1)$. Therefore, at all times in the algorithm, the pruning includes some node with a discrepancy $w$ which is the root of $G_1(i)$ for some $i \leq 2k$. It follows that the said greedy algorithm might never split any of the nodes which are the root of some $T_1$ under $T_{k-1}$, since these nodes also have a discrepancy of $w$. It also can never split $G_1(1)= T_1$, since this would require a pruning of size larger than $K = 2k$. As a result, such an algorithm might obtain a final pruning with a discrepancy of $kw$. In contrast, the pruning which includes all the child leaves of the sub-trees $T_1$ in $T_{k-1}$ and two child nodes of $G_1(2k)$ has a discrepancy of $w$. This gives an approximation factor of $k = K/2 \geq K/4$. 
\end{proof}

The next lemma shows that a different greedy approach, which selects the node to split by the maximal improvement in discrepancy, also fails. In fact, it obtains an unbounded approximation factor, even for a split quality as low as $1/2$. 
\begin{lemma}\label{lem:lookahead}
Consider a greedy algorithm that in each iteration splits the node $v$ in the current pruning that maximizes $\dev_v - (\dev_{v_R} + \dev_{v_L})$, where $v_R$ and $v_L$ are the child nodes of $v$. For any pruning size $K \geq 5$ and any value $N \geq 2$,  there exists an input tree such that the approximation factor of this algorithm is larger than $N$. For any $\epsilon > 0$, there exists such a tree with a split quality $q \leq \half + \epsilon$. 
\end{lemma}
\begin{proof}
  We define a hierarchical tree; see illustration in \figref{lookahead}. Let $w >
  0$. Its value will be defined below. Let $k \geq K-2$.   The input tree $T$ has two child nodes. The left child node, denoted $v_1$, has two children, $v_2$ and $v_3$. Each of these child nodes has two leaf children, one with weight $0$ and one with weight $Nw/2$. Thus, $\dev_{v_2} = \dev_{v_3} = Nw/2$, and $\dev_{v_1} = Nw$.

  The left child node is defined recursively as follows.   
  For an integer $m$, let $F(m)$ be some complete binary tree with $m$ leaves, each of weight $w' := w/3^k$. By definition, the discrepancy of the root of $F(m)$, for any integer $m$, is zero.
  We define $J(i)$ for $i \geq 0$ recursively, as follows. Let $J(0)$ be a tree such that its left child node is $F(1)$ and its right child node is the root of some complete binary tree with $3^k$ leaves of weight zero. Let $J(i)$ be a tree such that its left child is $F(2\cdot 3^{i-1})$ and its right child is $J(i-1)$. Note that $J(0)$  has $1=3^0$ leaf of weight $w'$, and by induction, $J(i)$ has $2\cdot 3^{i-1}+3^{i-1} = 3^i$ such leaves. In addition, $J(i)$ has $3^k$ leaves of weight $0$. Thus, the root of $J(i)$ has an average weight of $\dfrac{3^iw'}{3^i + 3^k} = \dfrac{w}{3^k + 3^{2k-i}}$, and a discrepancy of
 \[
   \alpha_i := 3^k \cdot \dfrac{w}{3^k + 3^{2k-i}} + 3^i (w' - \dfrac{w}{3^k + 3^{2k-i}}) = \dfrac{w}{1 + 3^{k-i}} + \frac{w}{3^{k-i}} - \dfrac{w}{3^{k-i} + 3^{2k-2i}}
 =    \dfrac{2w}{1 + 3^{k-i}}.
\]
The last equality follows by setting $a = 3^{k-i}$ and $b = 3^{2k - 2i}$ so that $\dfrac{w}{3^{k-i}} - \dfrac{w}{3^{k-i} + 3^{2k-2i}} = w(\frac{1}{a}-\frac{1}{a+b})$, and noting that 
\[
  \frac{1}{a} - \frac{1}{a+b} = \dfrac{b}{a(a+b)} = \dfrac{1}{a^2/b + a} = \frac{1}{1+3^{k-i}}.
  \]
	  Now, consider running the given algorithm on the input tree $T$ with pruning size $K$.
          Splitting $v_1$ into $v_2$ and $v_3$ does not reduce the total discrepancy. On the other hand, for any $i$, splitting the root of $J(i)$  reduces the total discrepancy, since it replaces a discrepancy of $\alpha_i$ with a discrepancy of zero (for $F(2\cdot 3^{i-1})$) plus a discrepancy of $\alpha_{i-1} < \alpha_i$ (for $J(i-1)$).  Therefore, the defined greedy algorithm will never split $v_1$, and will obtain a final pruning with a discrepancy of at least $\dev_{v_1} = Nw$. On the other hand, any pruning which includes the leaves under $v_2,v_3$ will have a discrepancy of at most that of $J(k)$, which is equal to $w$. Therefore, the approximation factor of the greedy algorithm is at least $N$.

         To complete the proof, we show that for a large $k$, the split quality of the tree is close to $1/2$. First, it is easy to see that the split quality of the tree rooted at $v_1$ is $1/2$. For the tree $J(k)$, observe that the discrepancy of $J(i)$ is $\alpha_i$ and the discrepancy of its child nodes is $0$ and $\alpha_{i-1}$. Since $\alpha_{i-1}/\alpha_i \leq 1/2$, $J(i)$ also has a split quality $\leq 1/2$, for all $i$. We have left to bound the ratio between the discrepancy of the root node and each of its child nodes. The left child node of the root has $4$ leaves of total weight $Nw$, and the right child node has $2\cdot 3^k$ leaves of total weight $3^k \cdot w' = w$. Therefore, the average weight of the root node is $\bar{w} := \dfrac{(N+1)w}{4 + 2\cdot 3^k}$. The discrepancy of the root node is thus
        \[
          2|Nw/2 - \bar{w}| + 2|\bar{w}| + 3^k|\bar{w} - w'| + 3^k|\bar{w}| = (N-1)w + 2\cdot 3^k\bar{w} = (N-1)w + \frac{(N+1)w}{2\cdot 3^{-k} + 1}.
          \]
          For $k \rightarrow \infty$, this approaches $2Nw$ from below. Thus, for any $\epsilon$, there is a sufficiently large $k$ such that the discrepancy of the root is at least $Nw/(\half + \epsilon)$. Since the discrepancy of each child node is $\leq Nw$, this gives a split quality of at most $\half + \epsilon$. 
      \end{proof}
      \begin{figure}[h]

	\begin{center}
		
		\begin{tikzpicture}[scale = 0.8]
                  \Tree [.{$T \,(\dev \geq 2Nw)$} [.{$v_1\,(\dev = Nw)$} [.{$v_2\,(\dev = Nw/2)$} {$0$} {$Nw/2$} ] [.{$v_3\,(\dev = Nw/2)$} {$0$} {$Nw/2$} ] ] {$J(3) \,(\dev = w)$} ];
                \end{tikzpicture}
                
                \vspace{2em}
                
                \begin{tikzpicture}[scale = 0.8]
                  \tikzset{level distance = 5em}
                \Tree [.{$J(3) \,(\dev = \dfrac{2w}{1+3^0} = w)$} {$F(2\cdot 3^2)$} [.{$J(2) \,(\dev = \dfrac{2w}{1+3^1} = w/2)$} {$F(2 \cdot 3^1)$}  [.{$J(1) \,(\dev = \dfrac{2w}{1+3^2} = w/5)$} {$F(2\cdot 3^0)$} [.{$J(0) \,(\dev = \dfrac{2w}{1+3^3} = w/14)$} {$F(1)$} {A tree with $3^k$ examples of weight $0$} ] ] ] ]
		\end{tikzpicture}
	\end{center}

	\caption{Illustrating the trees defined in the proof of \lemref{lookahead} for $k = 3$.}
	\label{fig:lookahead}
\end{figure}
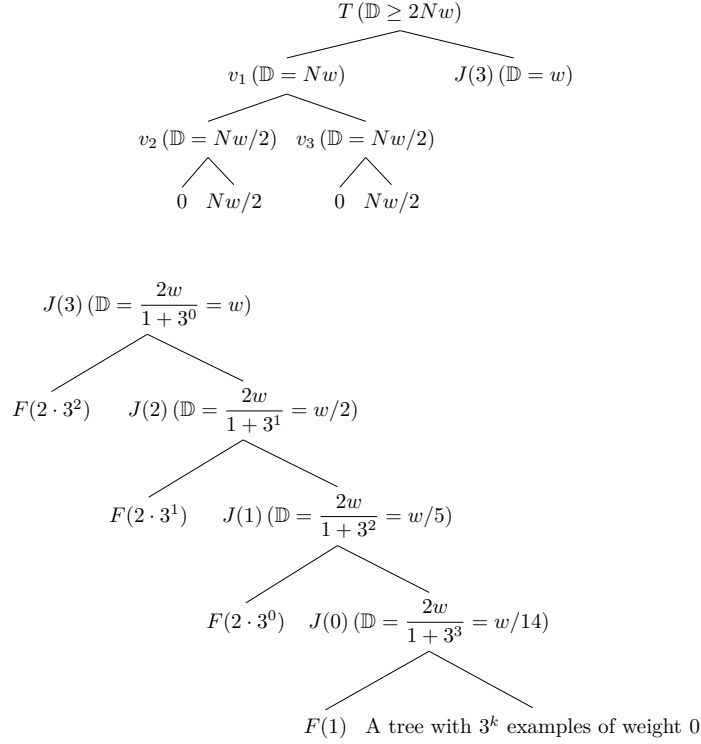

\section{Limitations of other discrepancy estimators}\label{app:estimator}
First, we show that the discrepancy cannot be reliably estimated from weight queries of examples alone, unless almost all of the weights are sampled. 
To see this, consider a node $v$ with $n+1$ descendant leaves (examples), all with the same weight $w = 1/(n+n^2)$, except for one special example with true weight either $n^2 w$ (first case) or $w$ (second case). In the first case, the average weight of the examples is $nw$, and $\dev_v = n\cdot|n w - w| + |nw - n^2w| = 2(n^2 - n)w = 2 -4/(n+1)$. In the second case, $\dev_v = 0$. However, in a random sample of size $\leq n/2$, the probability that the special example is not observed is $(1-1/(n+1))^{n/2} \geq \half$. When this example is not sampled, it is impossible to distinguish between the two cases unless additional information is available. This induces a large estimation error in this scenario.

Second, we show that even when $w^*_v$ is known, a naive empirical estimator
of the discrepancy can have a large estimation error. Recall that the
discrepancy of a node $v$ is defined as
\mbox{$\dev_v := \sum_{x \in \lfs_v} |w^*_v/N_v - w^*(x)|$}. Denote
$n := |\lfs_v|$. Given a sample $S_v$ of randomly selected examples in
$\lfs_v$ whose weight has been observed, the naive empirical estimator for the
discrepancy is $\frac{n}{|S_v|}\sum_{x \in S_v} |w^*_v/N_v - w^*(x)|$.  Now,
consider a case where $n-1$ examples from $\lfs_v$ have weight $0$, and 
a single example has weight $1$. We have
$\dev_v = (n-1) \cdot | 1/n -0| + |1/n - 1| = 2 - 2 /n$. However, if the heavy
example is not sampled, the naive empirical estimate is equal to
$1$. Similarly to the example above, if the sample size is of size $\leq n/2$,
there is a probability of more than half that the heavy example is not
observed, leading to an estimation error which is close to $1$.

\section{Proof of \lemref{prob}}\label{app:prob}
\begin{proof}[Proof of \lemref{prob}]
  Fix a node $v$ in $T$, and consider the value of $\hat{\dev}_v$ after drawing $M_v$ random samples from $\lfs_v$. To apply \lemref{bounds}, set $\bw$ to be the sequence of weights of the examples in $\lfs_v$. Then $\dev_v = \dev(\bw)$ and $\hat{\dev}(\bw) = \hat{\dev}_v$. For an integer $M$, let $\delta(M) := \frac{3 \delta}{K \pi^2 M^2}$. By \lemref{bounds}, with a probability at least $1-\delta(M_v)$,
  $|\dev_v - \hat{\dev}_v| \leq w_v^* \sqrt{2\ln(2/\delta(M_v))/M_v} \equiv \Delta_v$.
  We have $\sum_{n=1}^\infty \delta(n) = \frac{3\delta}{K\pi^2}\sum_{n = 1}^\infty \frac{1}{n^2} = \delta/(2K)$.
  Thus, for any fixed node $v$, with a probability at least $1-\delta/(2K)$, after any number of samples $M_v$,  $|\dev_v - \hat{\dev}_v| \leq \Delta_v$. Denote this event $D_v$. 

  Now, the pruning $P$ starts as a singleton containing the root node. Subsequently, in each update of $P$, one node is removed and its two children are added. Thus, in total $2K-1$ nodes are ever added to $P$  (including the root node).
  Let $v_i$ be the $i$'th node added to $P$, and let $V = \{v_1,\ldots,v_{2K-1}\}$. We have,
  \[
    \P[E_0] \geq \P[\forall v \in V, D_v] = 1-\P[\exists v \in V, \neg D_v].
    \]
    Now, letting $\cN$ be the nodes in $T$,
    \[
      \P[\exists v \in V, \neg D_v] \leq \sum_{i=1}^{2K-1} \sum_{v \in \cN}\P[v_i =v]\P[\neg D_v \mid v_i = v].
      \]
      Now, $\P[\neg D_v \mid v_i = v] = \P[\neg D_v]$, since the estimate $\hat{\dev}_v$ uses samples that are drawn after setting $v_i = v$. Therefore, $\P[\neg D_v \mid v_i = v]\leq \delta/(2K)$. Since $\sum_{v \in \cN}\P[v_i =v] = 1$, it follows that $\P[\exists v \in V, \neg D_v]\leq \delta$. Therefore,
      We have $\P[E_0] \geq 1-\delta$, as claimed.
      
\end{proof}

\section{Proof of \lemref{sons}}\label{app:clowerbound}

\begin{proof}[Proof of \lemref{sons}]
  To prove the first part, denote the nodes in $P$ by $v_1,\ldots,v_n$ and let $v_0 := r$ be the root node. For $i \in \{0,\ldots,n\}$, let $\bw_i$ be a sequence of length $N_{v_i}$ of the weights $w^*(x)$ of all the leaves $x \in \lfs_{v_i}$. Let $\bar{w}^*_i := \frac{w^*_{v_i}}{N_{v_i}}$.  Then 
	\begin{align*}
          \dev_{v_i} & = \normone{\bw_i -   \bar{w}^*_i \cdot \bone}
                       = \normone{\bw_i -   \bar{w}^*_0 \cdot \bone +  \bar{w}^*_0\cdot \bone -  \bar{w}^*_i\cdot \bone}\\
                       &\le \normone{\bw_i -   \bar{w}^*_0 \cdot \bone} +  \normone{\bar{w}^*_0\cdot \bone -  \bar{w}^*_i\cdot \bone} = \normone{\bw_i - \bar{w}^*_0 \cdot \bone} + N_{v_i} | \bar{w}^*_i - \bar{w}^*_0|.
	\end{align*}
        Now, observe that 
        \[
        N_{v_i}| \bar{w}_i^* -  \bar{w}^*_0| = | w_i^*  - N_{v_i} \bar{w}^*_0 | = | \sum_{x \in \lfs_{v_i}} (w^*(x) - \bar{w}^*_0) | \le \sum_{x \in \lfs_{v_i}} |w^*(x) - \bar{w}^*_0 | = \normone{\bw_i - \mathbf{1} \cdot \bar{w}^*_0}.
      \]
      Therefore, $\dev_{v_i} \le 2 \normone{\bw_i - \mathbf{1} \cdot \bar{w}^*_0}$.
      Summing over all the nodes, and noting that $\bw_0 = \bw_1\circ \ldots \circ \bw_n$, we get:
      \[
        \sum_{i\in [n]} \dev_{v_i} \le \sum_{i\in [n]} 2 \normone{\bw_i - \bar{w}^*_0 \cdot \bone} = 2 \normone{\bw_0 -  \bar{w}^*_0 \cdot \bone} = 2 \dev_{v_0},
      \]
      which proves the first part of the lemma.

          For the second part of the lemma, let $n$ be an integer sufficiently large such that $\frac{2}{1+2/n} \geq 2-\epsilon$. We consider a tree (see illustration in \figref{splitproof}) with $n+2$ leaves. Denote the leaves by $x_1,\ldots,x_{n+2}$. Denote $w := 1 /(n+2)$, and define $w^*(x_1) = 0$, $w^*(x_2) = 2w$, and $w^*(x_i) = w$ for $i \geq 2$. 
  Denote by $v_0$ the root node of the tree, and let its two child nodes be $v_1$ and $v_2$. The tree is organized so that $x_1$ and $n/2$ of the examples with weight $w$ are descendants of $v_1$, and the other examples are descendants of $v_2$. $v_1$ has two child nodes, one is the leaf $x_1$ and the other is some binary tree whose leaves are all the other $n/2$ examples. Similarly, $v_2$ has a child node which is the leaf $x_2$, and the other examples are organized in some  binary tree rooted at the other child node.

  It is easy to see that $\dev_{v_0} = 2w$. To calculate $\dev_{v_1}$, note that the average weight of node $v_1$ is $\frac{n w/2}{n/2+1} = nw^2$. Thus,
	\[
	\dev_{v_1} \equiv \sum_{x \in \lfs_{v_1}} |nw^2 -w^*(x)|  = \frac{n}{2}(w - nw^2) + nw^2 = \frac{nw}{2}(1 - \frac{n}{n+ 2}) + nw^2 =	2 n w^2.
      \]
      A similar calculation shows that $\dev_{v_2} = 2 n w^2$. Define the pruning $P = \{v_1,v_2\}$ of the tree rooted at $v_0$. Then
      \[
        \dev_P = 4nw^2 = 2nw \cdot \dev_{v_0} = \frac{2}{1+2/n}\dev_{v_0} \geq (2-\epsilon)\dev_{v_0},
      \]
      as required. 

      To show that this tree has a split quality of less than $1$, note that $\dev_{v_1} = \dev_{v_2} = \frac{1}{1+2/n}\dev_{v_0}$, and that the discrepancy of each of the child nodes of $v_1$ and $v_2$ is zero, since all their leaves have the same weight. Therefore, this tree has a split quality $\frac{1}{1+2/n} < 1$.
      \end{proof}

      \begin{figure}
        \begin{center}
      \begin{tikzpicture}
        \Tree [.{$v_0$} [.{$v_1$} {$0$} [.{.} [.{.} {$w$} {$w$} ] [.{.} {$w$} {$w$} ] ] ] [.{$v_2$} {$2w$} [.{.} [.{.} {$w$} {$w$} ] [.{.} {$w$} {$w$} ] ] ] ]
      \end{tikzpicture}
    \end{center}
    \caption{Illustrating the tree constructed in the proof of the second part of \lemref{sons}, for $n = 8$.}
    \label{fig:splitproof}
  \end{figure}
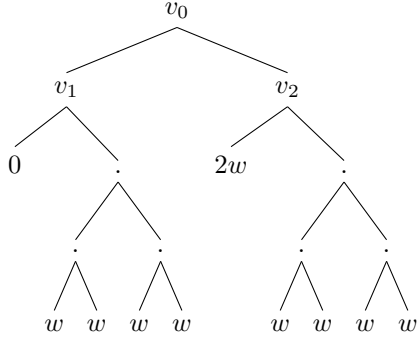

  \section{Proof of \lemref{devbound}}\label{app:prooflem44}
  
  \begin{proof}[Proof of \lemref{devbound}]
	
  Let $Q$ be some pruning such that $|Q| = K$.  Partition $\pout$ into $R,P_a$ and $P_d$, where $R := \pout \cap Q$, $P_a \subseteq \pout$ is the set of strict ancestors of nodes in $Q$, and $P_d \subseteq \pout$ is the set of strict descendants of nodes in $Q$.  Let $Q_a \subseteq Q$  be the ancestors of the nodes in $P_d$ and let $Q_d\subseteq Q$ be the descendants of the nodes in $P_a$, so that $R,Q_d$ and $Q_a$ form a partition of $Q$. First, we prove that we may assume without loss of generality that $P_a,P_d,Q_a,Q_d$ sets are non-empty. 

  \textbf{Claim 1}: 
  If any of the sets $P_a,P_d,Q_a,Q_d$ is empty then the statement of the lemma holds.

  \textbf{Proof of Claim 1:} Observe that if any of $P_a,P_d,Q_a,Q_d$ is empty then all of these sets are empty: By definition, $P_a = \emptyset \Leftrightarrow Q_d = \emptyset$ and $P_d = \emptyset \Leftrightarrow Q_a = \emptyset$. Now, suppose that $P_a = Q_d = \emptyset$. Since $|R| + |P_d| + |P_a| = |\pout| = |Q| = |R| + |Q_d| + |Q_a|$, we deduce that $|P_d| = |Q_a|$. But for each node in $Q_a$, there are at least two descendants in $P_d$, thus $|P_d|\geq 2|Q_a|$. Combined with the equality, it follows that $P_d = Q_a  = \emptyset$. The other direction is proved in an analogous way. Now, if $P_a=P_d=Q_a=Q_d = \emptyset$ then $\pout = Q$, thus in this case $\dev_P = \dev_Q$, which means that the statement of the lemma holds. This concludes the proof of Claim 1. 

Assume henceforth that $P_a,P_d,Q_a,Q_d$ are non-empty.  Let $r$ be the node with the smallest discrepancy out of the nodes that were split by \algname\ during the entire run. Define $\theta := |P_a|\cdot \dev_r/\dev_{Q_a}$ if $\dev_{Q_a}>0$  and $\theta := 0$ otherwise.

  \textbf{Claim 2}: $\dev_{\pout} \leq \max(2,\beta \theta) \dev_Q.$

  \textbf{Proof of Claim 2}: 
	 We bound the discrepancies of $P_d$ and of $P_a$ separately.
	 For each node $u\in Q_a$, denote by $P(u)$ the descendants of $u$ in $P_d$. These form a pruning of the sub-tree rooted at $u$. In addition, the sets $\{P(u)\}_{u \in Q_a}$ form a partition of $P_d$. Thus, by the definition of discrepancy and \lemref{sons}, 
	 \begin{equation}\label{eq:pd}
		\dev_{P_d} = \sum_{u \in Q_a} \dev_{P(u)} \le \sum_{u \in Q_a} 2 \dev_u = 2 \dev_{Q_a}.
              \end{equation}
              Let $P$  be the pruning when \algname\ decided to split node $r$. By the definition of the splitting criterion $\spc$ (\eqref{split}), for all $v \in P\setminus \{r\}$, at that time it held that $\beta(\hat{\dev}_r - \Delta_r) \ge \hat{\dev}_{v} + \Delta_{v}$. Since $E_0$ holds, we have $\dev_r \geq \hat{\dev}_r - \Delta_r$ and $\hat{\dev}_{v} + \Delta_{v} \geq \dev_v$. Therefore,  $\forall v \in P\setminus \{r\}, \beta \dev_r \geq \dev_v$.

              Now, any node $v'\in \pout \setminus P$ is a descendant of some node $v \in P$. Since $T$ has split quality $q$ for $q<1$, we have $\dev_{v'} \leq \dev_v$. Therefore, for all $v' \in \pout$, $\dev_{v'} \leq \beta \dev_r$. In particular, $\dev_{P_a} \equiv \sum_{v\in P_a} \dev_v \leq \beta |P_a|\dev_r$. Since all nodes in $Q_a$ were split by \algname\ $\dev_{Q_a} = 0$ implies $\dev_r = 0$, therefore in all cases $\dev_{P_a} \leq \beta \theta \dev_{Q_a}$. Combining this with \eqref{pd}, we get that
    \begin{equation*}
    \begin{split}
    \dev_{\pout} & = \dev_{R} +\dev_{P_d}  + \dev_{P_a} \le \dev_R +  2 \dev_{Q_a}+ \beta \theta \dev_{Q_a} \\ & \leq \max(2,\beta \theta) \dev_Q,
    \end{split}
\end{equation*}
which completes the proof of Claim 2.

It follows from Claim 2 that to bound the approximation factor, it suffices to bound $\theta$. Let $P_d'$ be the set of nodes both of whose child nodes are in $P_d$ and denote
$n:=|P_d'|$. In addition, define
\[
  \alpha := \frac{\log(1/q)}{\log(|P_a|)+\log(1/q)} \leq 1.
\]
We now prove that $\theta \leq 2/\alpha$ by considering two complementary cases, $n \geq \alpha|P_a|$ and $n < \alpha|P_a|$. The following claim handles the first case. 

\textbf{Claim 3:} if $n \geq \alpha|P_a|$, then $\theta \leq 2/\alpha$.

\textbf{Proof of Claim 3}: 
Each node in $P_d'$
has an ancestor in $Q_a$, and no ancestor in $P_d'$. Therefore, $P_d'$ can be partitioned to subsets
according to their ancestor in $Q_a$, and each such subset is a part of some
pruning of that ancestor. Thus, by \lemref{sons}, 
$\dev_{P_d'} \leq 2 \dev_{Q_a}$. Hence, for some node $v \in P'_d$, $\dev_v \leq 2 \dev_{Q_a}/n$. It follows from the definition of $r$ that
$\dev_r \leq 2 \dev_{Q_a}/n$. Hence, $\theta \leq 2|P_a|/n$. Since $n \geq \alpha|P_a|$, we have $\theta \leq 2/\alpha$ as claimed.

We now prove this bound hold for the case $n < \alpha |P_a|$. For a node $v$ with an ancestor in $Q_a$, let $l_v$ be the path length from this ancestor to $v$, and define $L := \sum_{v \in P'_d} l_v$. We start with an auxiliary Claim 4, and then prove the required upper bound on $\theta$ in Claim 5.

\textbf{Claim 4}: $L \geq |P_a| - n$.

        \textbf{Proof of Claim 4}: Fix some $u \in Q_a$, and let $P_u(t)$ be the set of nodes in the pruning $P$ in iteration $t$ which have $u$ as an ancestor. Let $P'_u(t)$ be the set of nodes both of whose child nodes are in $P_u(t)$, and denote $L_u(t) := \sum_{v \in P'_u(t)} l_v$. We prove that for all iterations $t$, $L_u(t) \geq |P_u(t)|-2$.
        First, immediately after $u$ is split, we have $P'_u(t) = \{u\}$, $|P_u(t)| = 2$, $L_u(t) = 0$. Hence, $L_u(t) \geq |P_u(t)|-2$.
        Next, let $t$ such that $P_u(t)$ grows by 1, that is some node $u_t$ in $P_u(t)$ is split. If $u_t$ is the child of a node $v_t \in P'_u(t)$, then $P'_u(t+1) = P'_u(t) \setminus \{v_t\} \cup \{u_t\}$.  In this case, $L_u(t+1) = L_u(t) + 1$, since $l_{u_t} = l_{v_t} +1$. Otherwise, $u_t$ is not a child of a node in $P'_u(t)$, so $P'_u(t+1) = P'_u(t) \cup \{u_t\}$, and so $L_u(t+1) = L_u(t) + l_{u_t} \geq L_u(t) + 1$. Thus, $L_u(t)$ grows by at least $1$ when the size of $P_u(t)$ grows by $1$. It follows that in all iterations, $L_u(t) \geq |P_u(t)|-2$. Summing over $u \in Q_a$ and considering the final pruning, we get $L \geq |P_d| - 2|Q_a|$. 
        Now, since $|P| = |Q|$, we have $|P_d| - |Q_a| = |Q_d| - |P_a|$. From the definition of $Q_d$, $|Q_d| \geq 2|P_a|$. Therefore, $|P_d| - |Q_a| \geq |P_a|$. It follows that $L \geq |P_a| - |Q_a|$. Lastly, every node $u \in Q_a$ was split by \algname, and has at least one descendant in $P_d'$. Therefore, $|Q_a| \leq |P'_d| \equiv n$. Hence, $L \geq |P_a| - n$, which concludes the proof of Claim 4.

        \textbf{Claim 5}: if $n < \alpha|P_a|$, then $\theta \leq 2/\alpha$.

\textbf{Proof of Claim 5}: 
It follows from Claim 4 that for some node $v \in P_d'$, $l_v \geq (|P_a| - n)/n = |P_a|/n-1 > 0$, where the last inequality follows since $n < \alpha |P_a| < |P_a|$. Letting $u \in Q_a$ be the ancestor of $v$ in $Q_a$, we have by the split quality  $q$ of $T$ that $\dev_v \leq \dev_u\cdot q^{\frac{|P_a|}{n}-1}$. Since $u \in Q_a$, we have $\dev_u \leq \dev_{Q_a}$. In addition, $\dev_r \leq \dev_v$ by the definition of $r$. Therefore, $\dev_r \leq \dev_{Q_a}\cdot q^{\frac{|P_a|}{n}-1}$. 
Since $n < |P_a|\alpha$ and $q<1$, from the definition of $\alpha,\theta$ we have
\[
  \theta \le |P_a|q^{\frac{|P_a|}{n}-1}\leq 1 \leq 2/\alpha.
\]
This proves Claim 5.

Claims 3 and 5 imply that in all cases, $\theta \leq 2/\alpha$. By Substituting $\alpha$, we have that
\[
  \theta \leq 2 (\frac{\log(|P_a|)}{\log(1/q)}+1) \leq 2 (\frac{\log(K)}{\log(1/q)}+1).
  \]
Placing this upper bound in the statement of Claim 2 concludes of the lemma. 
\end{proof}

\section{An auxiliary lemma}\label{app:mathineq}

\begin{lemma}\label{lem:mathineq}
  Let $\mu > 0, \phi \ge 0, p \geq 0$. 
  If $p < \phi \ln(\mu p)$ then $p < e \phi \ln(e \mu \phi)$.
\end{lemma}
\begin{proof}
  We assume that $p \geq e \phi \ln(e \mu \phi)$ and prove that
  $p \geq \phi \ln(\mu p)$.
  First, consider the case $\mu \phi< 1$. In this case, $p/\phi> \mu p \geq \ln(\mu p)$. Therefore, $p \geq \phi\ln(\mu p)$.   
  Next, suppose $\mu \phi\geq 1$. Define the function $f(x) := x/\log(\mu x)$, and note that it is monotone increasing for $x \ge e / \mu$. By the assumption, we have $p \geq e \phi \ln(e \mu \phi)$. In addition, $\phi \geq 1/\mu$, hence $e \phi \ln(e \mu \phi) \geq e\phi \geq  e/\mu$. 
Therefore, $f(p) \ge f(e \phi \ln(e \mu  \phi))$, and we can conclude that 
\[
\frac{p}{\ln(\mu p)} \equiv f(p) \geq f(e \phi \ln(e \mu  \phi)) \equiv \frac{e\phi \ln(e \mu \phi)}{\ln(e \mu \phi \ln(e \mu \phi))}\geq \frac{e\phi \ln(e \mu \phi)}{2\ln(e \mu \phi)} \ge e\phi / 2 \ge \phi.
\]
Note that we used the fact $\ln(x\ln(x)) \leq 2\ln(x)$, which follows since for any $x$, $\ln(x) \leq x$. This proves the claim.
\end{proof}

\section{Tightening $\Delta_v$ using empirical Bernstein bounds}\label{app:bernstein}
We give a tighter definition of $\Delta_v$, using the empirical Bernstein bound of \citet{MaurerPo09}. This tighter definition does not change the analysis, but can improve the empirical behavior of the algorithm, by allowing it to require  weight queries of fewer examples in some cases. 
The empirical Bernstein bound states that for i.i.d.~random variables $Z_1,\ldots,Z_m$ such that $\P[Z_i \in [0,1]] = 1$, with a probability $1-\delta$,
\begin{equation}\label{eq:bernstein}
  |\E[Z_1] - \frac{1}{m}\sum_{i \in [m]}Z_i| \leq \sqrt{8V_m \ln(2/\delta)/m} + 14\ln(2/\delta)/(3(m-1)),
\end{equation}
where $V_m := \frac{1}{m(m-1)}\sum_{1 \leq i < j \leq m}(Z_i - Z_j)^2$. 
The following lemma derives the resulting bound. The proof is similar to the proof of \lemref{bounds}, except that it uses the bound above instead of Hoeffding's inequality.
\begin{lemma}\label{lem:boundsB}
  Consider the same definitions and notations as in \lemref{bounds}.
        Let
      \[
        V := \frac{1}{m(m-1)}\sum_{1 \leq i < j \leq m} (|Z_i - W|-Z_i - |Z_j - W|+Z_j)^2.
        \]
	Then, with a probability at least $1-\delta$, 
	\[
          |\dev(\bw) - \hat{\dev}(\bw)| \le 	n\sqrt{8 V \ln(2/\delta)/m}   + 28\norm{\bw}_1\ln(2/\delta)/(3(m-1)).
          \]
	
\end{lemma}
\begin{proof}
		Let $Z_i' = |Z_i - W| - Z_i$. If $Z_i \geq W$, then $Z_i' = W$. Otherwise, we have $Z_i \leq W$, in which case $Z_i' = W - 2Z_i \geq -W$. Therefore, $\P[Z_i' \in [-W,W]] = 1$. Thus, applying \eqref{bernstein} and normalizing by $2W$, we get that with a probability $1-\delta$, 
	\[
	|\E[Z'_1] - \frac{1}{m}\sum_{i\in [m]}Z'_i| \leq \sqrt{8 V \ln(2/\delta)/m}   + 28W\ln(2/\delta)/(3(m-1)).
	\]
	
	Now, $\E[Z'_1] = \frac{1}{n}(\normone{\bw - W\cdot \mathbf{1}} - \normone{\bw}) = \frac{1}{n}(\dev(\bw)-\normone{\bw})$. In addition,
	\[
	\frac{1}{m}\sum_{i\in [m]}Z'_i = \frac{1}{m}(\normone{\bZ - W\cdot \bone}-\normone{\bZ}) = \frac{1}{n}(\hat{\dev}(\bw)-\normone{\bw}).
	\]
	Therefore, with a probability at least $1-\delta$,
	\[
	\big|\dev(\bw) - \normone{\bw} - (\hat{\dev}(\bw)-\normone{\bw})\big|  \leq n\sqrt{8 V \ln(2/\delta)/m}   + 28nW\ln(2/\delta)/(3(m-1)).
	\]
	By noting that $nW = \normone{\bw}$, this completes the proof.
\end{proof}
The tighter definition of $\Delta_v$ is obtained by taking the minimum between this bound and the one in \lemref{bounds}. Thus, we set
\[
  \Delta_v := \min(w^*_v \cdot \sqrt{2 \ln(2K \pi^2 M_v^2/ (3\delta))/M_v}, N_v\sqrt{8 V \ln(2/\delta)/M_v}   + 28w_v^*\ln(2/\delta)/(3(M_v-1))).
\]
The entire analysis is satisfied also by this new definition of $\Delta_v$. Its main advantage is obtaining a smaller value when $V$ is small. This may reduce the number of weight queries required by the algorithm in some cases.

\clearpage
\section{Full experiment results}\label{app:exps}

In this section, we provide the full results and details of all the experiments described in \secref{exps}. A python implementation of the proposed
algorithm and of all the experiments can be found at \url{https://github.com/Nadav-Barak/AWP}. For each experiment, we report the average normalized output
distance over $10$ runs, as a function of the pruning size. Error bars, represented by shaded regions, represent the maximal and minimal normalized disrepancies obtained in these runs. Note that the error bars are sometimes too small to observe, in cases where the algorithms behave deterministically or very similarly in different runs of the same experiment. 

\figref{adult} provides the full results for the experiments on the Adult data set. We give here more details on the procedure which we used to create the hierarchical tree: we started with a tree that includes only the root node, and then iteratively selected a random node to split and a random attribute to use for the split. 
For numerical attributes, the split was based on a threshold corresponding to the median value of the attribute. For discrete attributes, the attribute values were divided so that the split is fairly balanced. 
We generated several target distributions by partitioning the data set into ordered bins, where in each experiment the partition was based on the value of a different attribute. The tested attributes were all discrete attributes with a small number of possible values: ``occupation'', ``relationship'', ``marital status'' and ``education-num''. For the last attribute, all values up to $8$ were mapped to a single bin and similarly for all values from $14$ and above, to avoid very small bins. We then allocated the target weight to each bin so that each example in a given bin is $N$ times more heavier than each example in the next bin. We tested $N=2,4$, which appear in the left and right columns of \figref{adult}, respectively. 
It can be seen in \figref{adult}, that except for a single configuration ($N=2$ and the ``relationship'' attribute), \algname\ always performs better than the baselines.

We now turn to the visual data sets. In all these data sets except for MNIST, images were resized to a standard $224\times 224$ size and transformed to grayscale.
\figref{bright} provides the full results for first experiment on the MNIST and Caltech256 data sets. In this experiment,  the examples were divided into $10$ bins by image brightness, and weights were allocated such that the weight of an example is $N$ times heavier than an
          example in the next bin. The plots show results for $N=2$ (left) and for $N=4$ (right). The top row gives the results for MNIST and the bottom row gives the results for Caltech256. Here too, it can be seen that \algname\ obtains significantly better approximations of the target.

          \figref{mnistrand} provides the results for the MNIST data set for bins allocated by class, using the same scheme of weight allocation for each bin as in the previous experiment.
Results for $N=2$ (left column) and $N=4$ (right column) are reported for three random bin orders. 
\figref{caltechrand} provides the results of an analogous this experiment for the Caltech256 data set. For this data set, the 10 bins were generated by randomly partitioning the $256$ classes into 10 bins with (almost) the same number of classes in each. The allocation of classes to bins and their ordering, for both data sets, are provided as part of the submitted code. It can be seen that \algname\ obtains an  improvement over the baselines in the MNIST experiments, while the Caltech256 experiment obtains about the same results for all algorithms, with a slight advantage for \algname. 

\figref{nn} and \figref{BingSupercat} provides the results of the experiments with the data set pairs. In all experiments, the input data set was Caltech256. In each experiment, a different target data set was fixed. The weight of each Caltech256 example was set to the fraction of images from the target data set which have this image as their nearest neighbor. The target data sets were the Office dataset \citep{saenko2010adapting}, out of which the 10 classes that also exist in Caltech256 (1410 images) were used, and The Bing dataset \citep{BergamoTorresani10,Bing}, which includes $300$ images in each Caltech256 class. For Bing, we also ran three experiments where images from a single super-class from the taxonomy in \citet{griffin2007caltech} were used as the target data set. The super-classes that were tested were ``plants'' ``insects'' and ``animals''. The classes in each such super-class, as well as those in the Office data set, are given in \tabref{cat}. In these experiments as well, the advantage of \algname\ is easily observed. 

\begin{figure}[h]
	\includegraphics[width=0.5\textwidth]{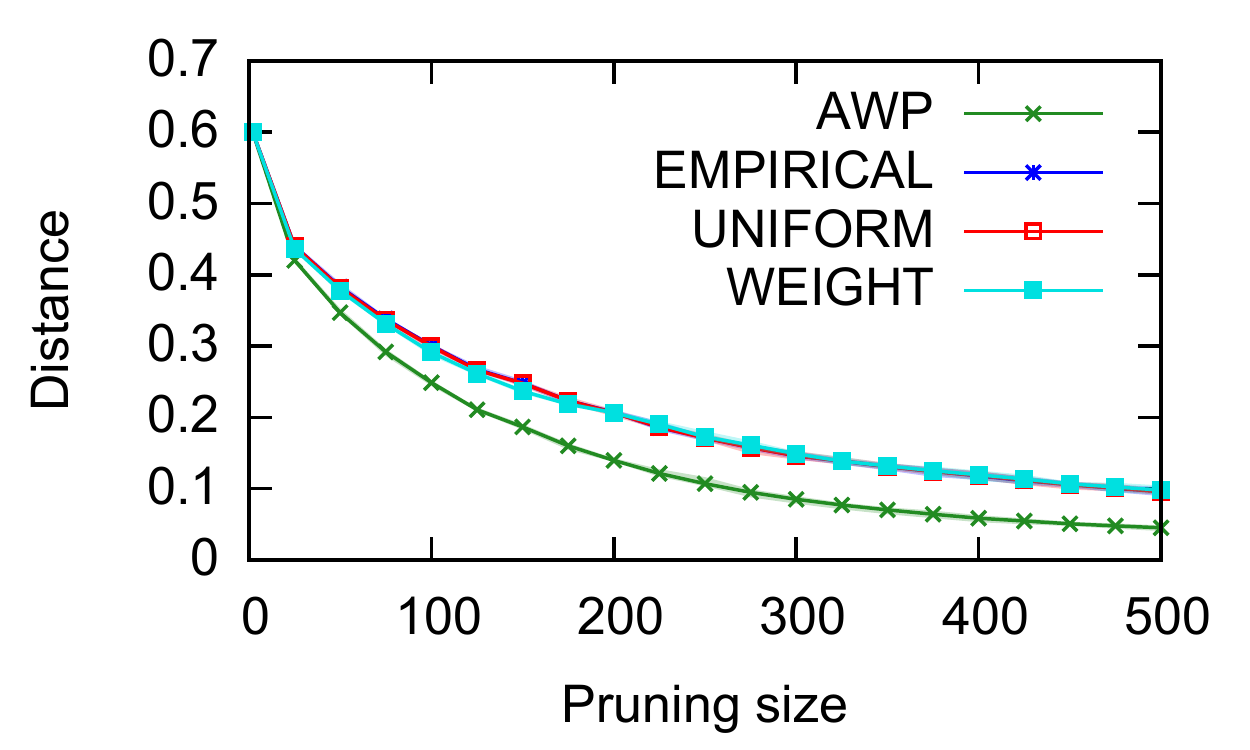}
	\includegraphics[width=0.5\textwidth]{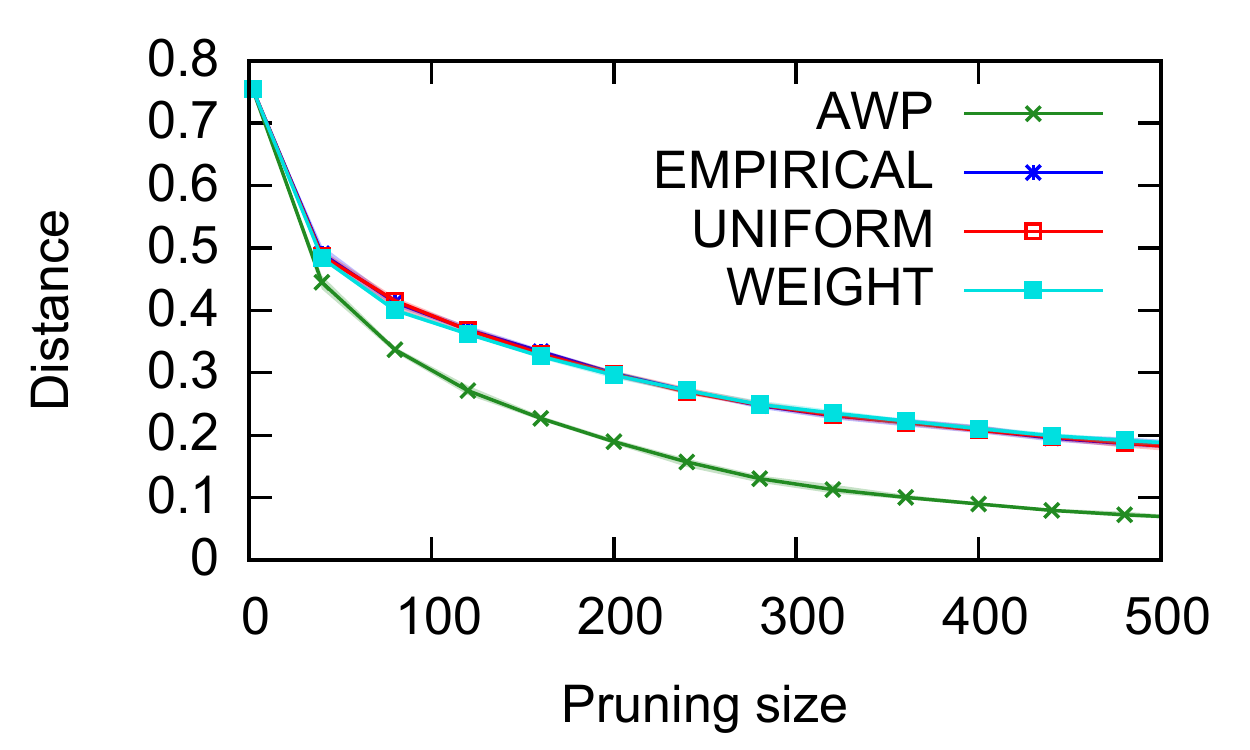} \\
	\includegraphics[width=0.5\textwidth]{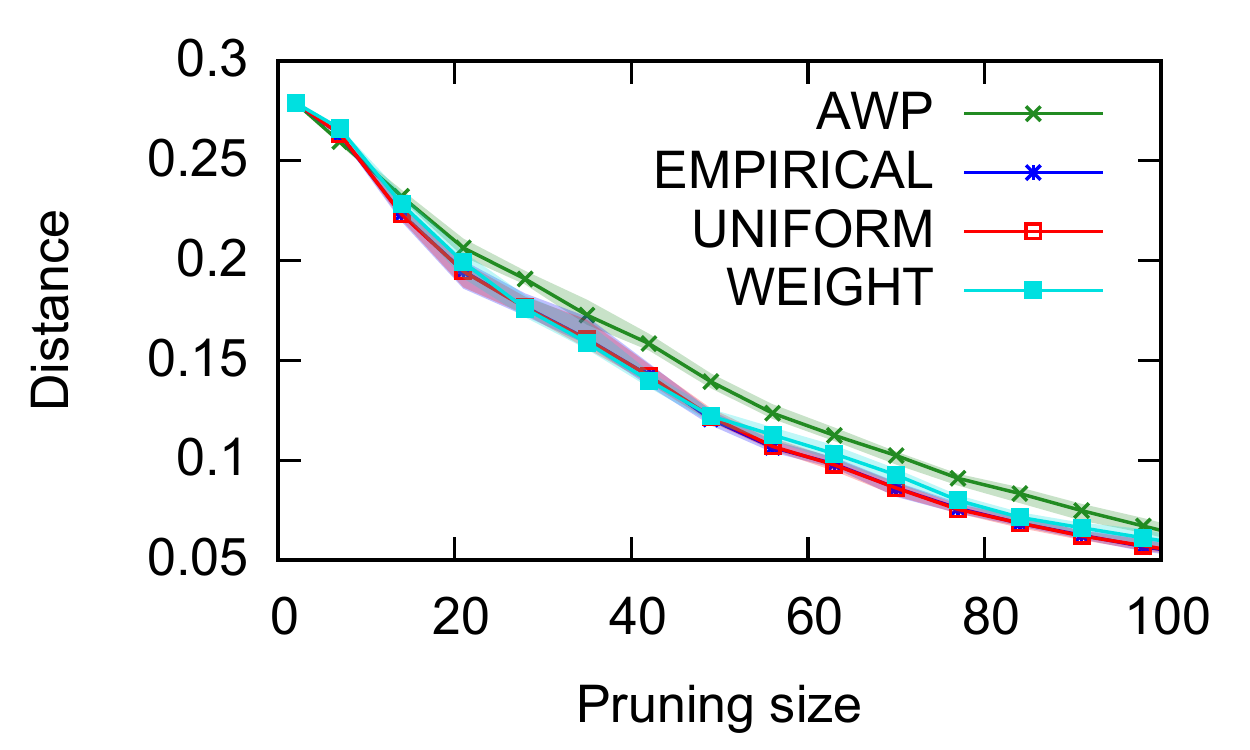}
	\includegraphics[width=0.5\textwidth]{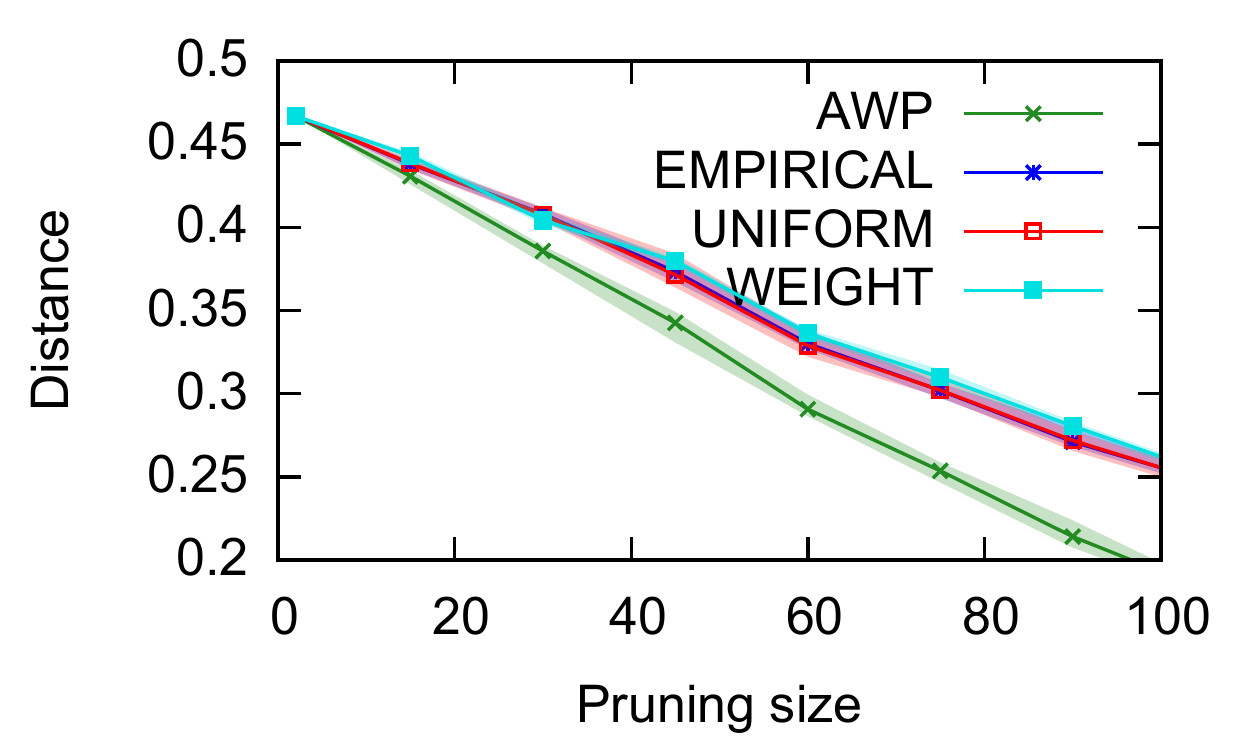} \\
	\includegraphics[width=0.5\textwidth]{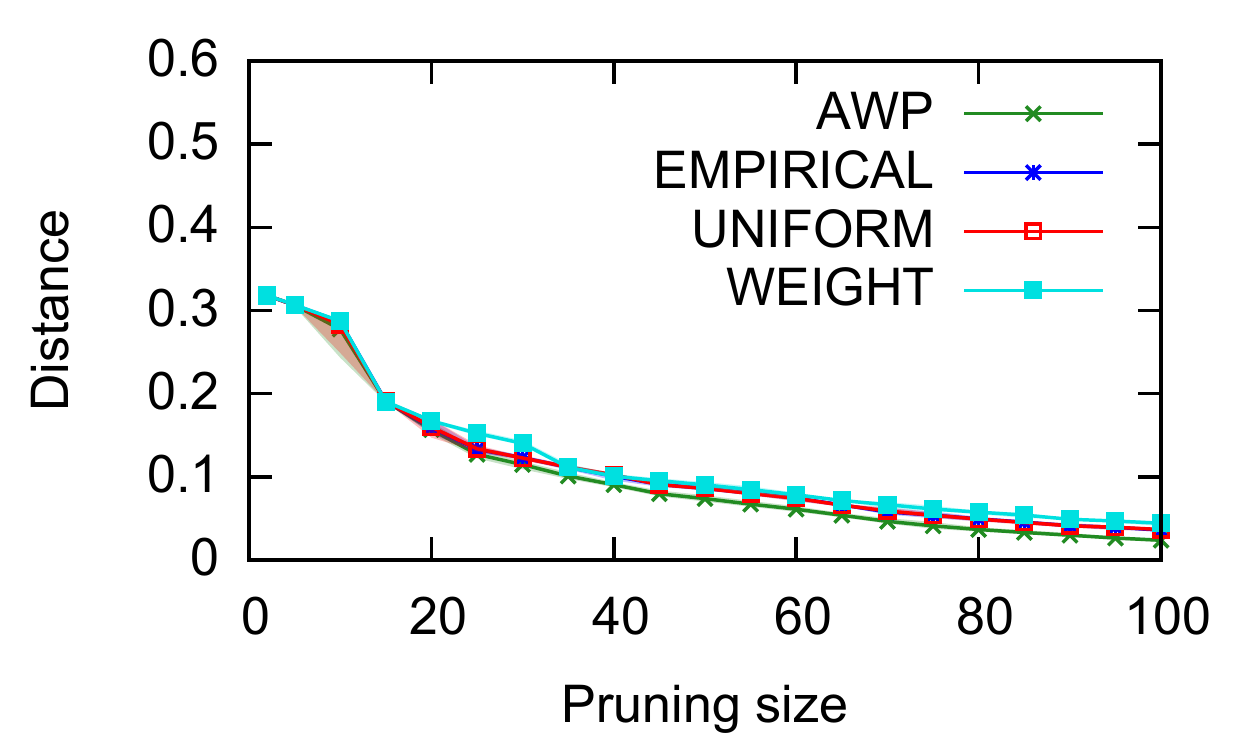}
	\includegraphics[width=0.5\textwidth]{Plots_new/AdultCensus_marital-statusN4.pdf}\\
        	\includegraphics[width=0.5\textwidth]{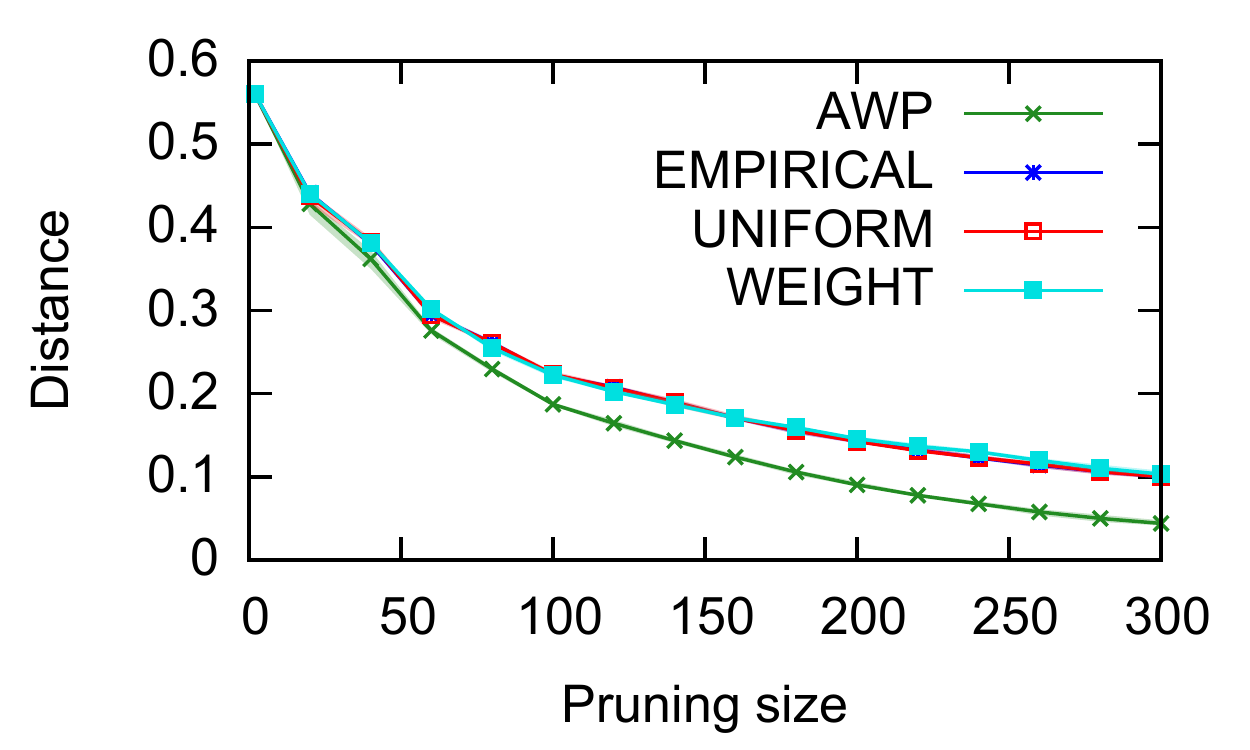}
	\includegraphics[width=0.5\textwidth]{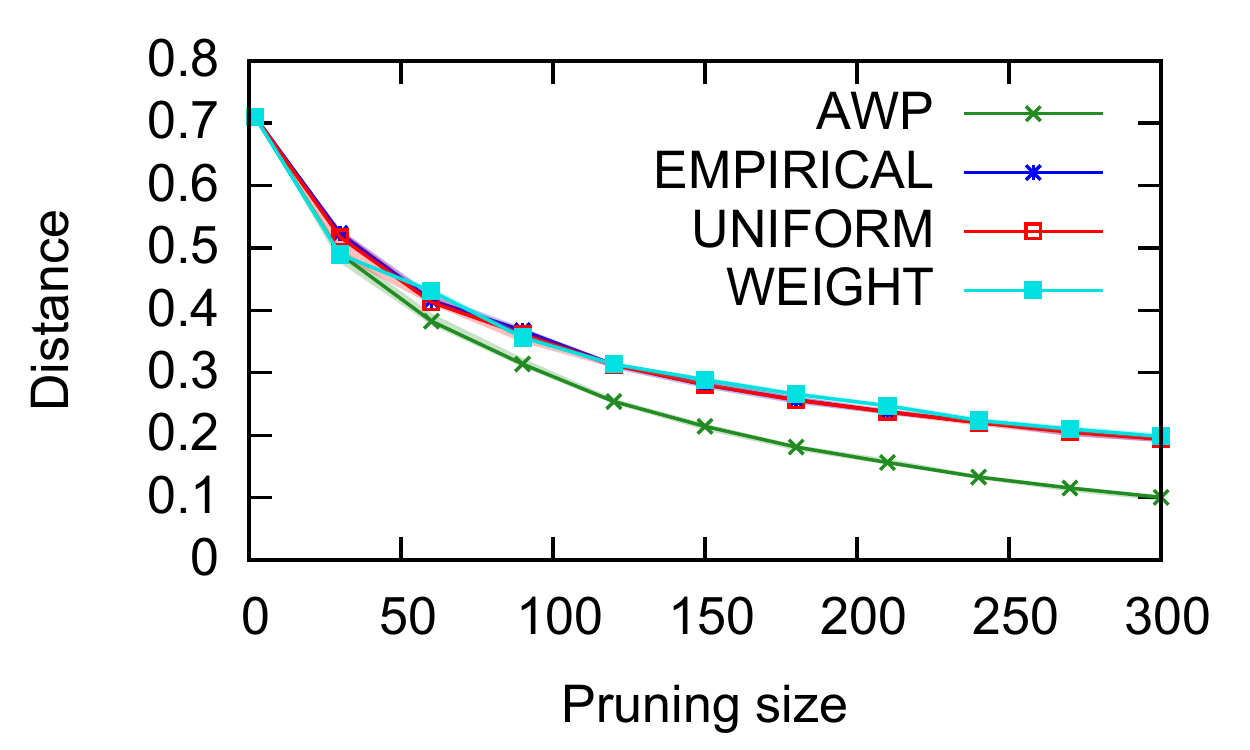}. 
	\caption{The Adult dataset experiments. Each row report the results for experiments on a different parameter in the following order: ``occupation'', ``relationship'', ``marital-status'', ``education-num''. Left: $N=2$. Right: $N=4$.}
	\label{fig:adult}
\end{figure}

\begin{figure}[h]
  \includegraphics[width=0.5\textwidth]{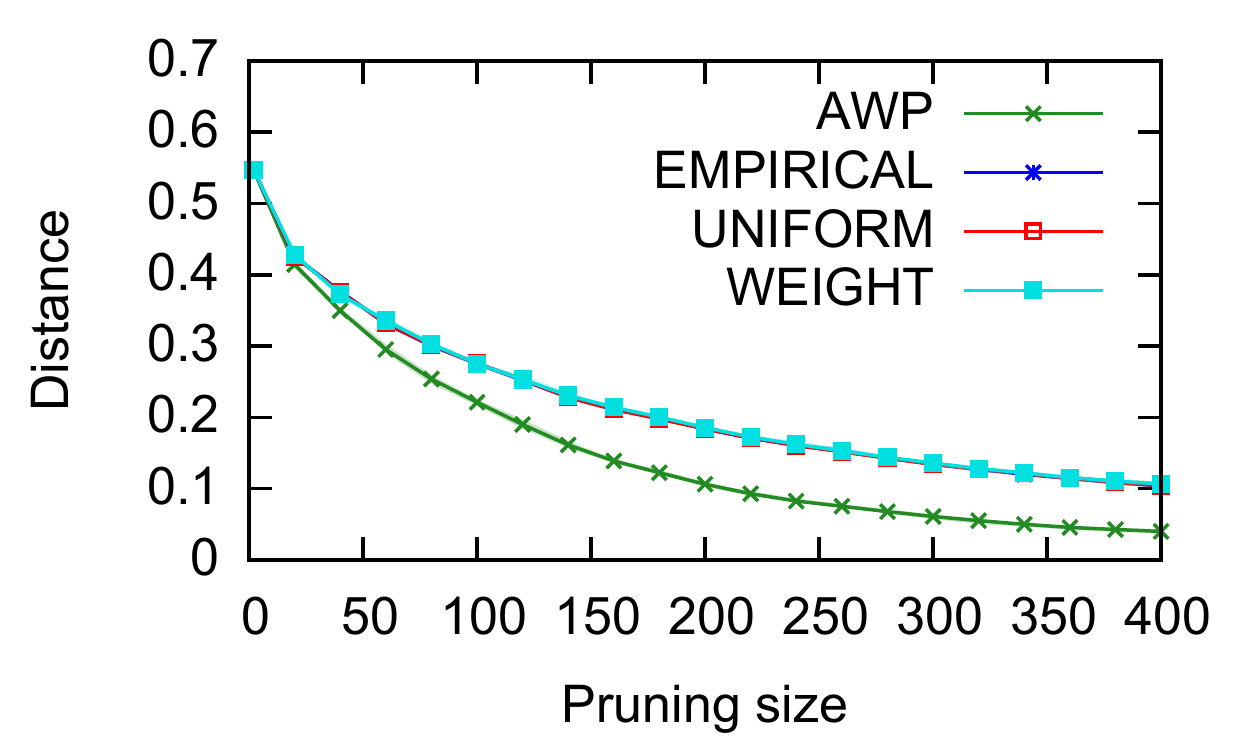}
  \includegraphics[width=0.5\textwidth]{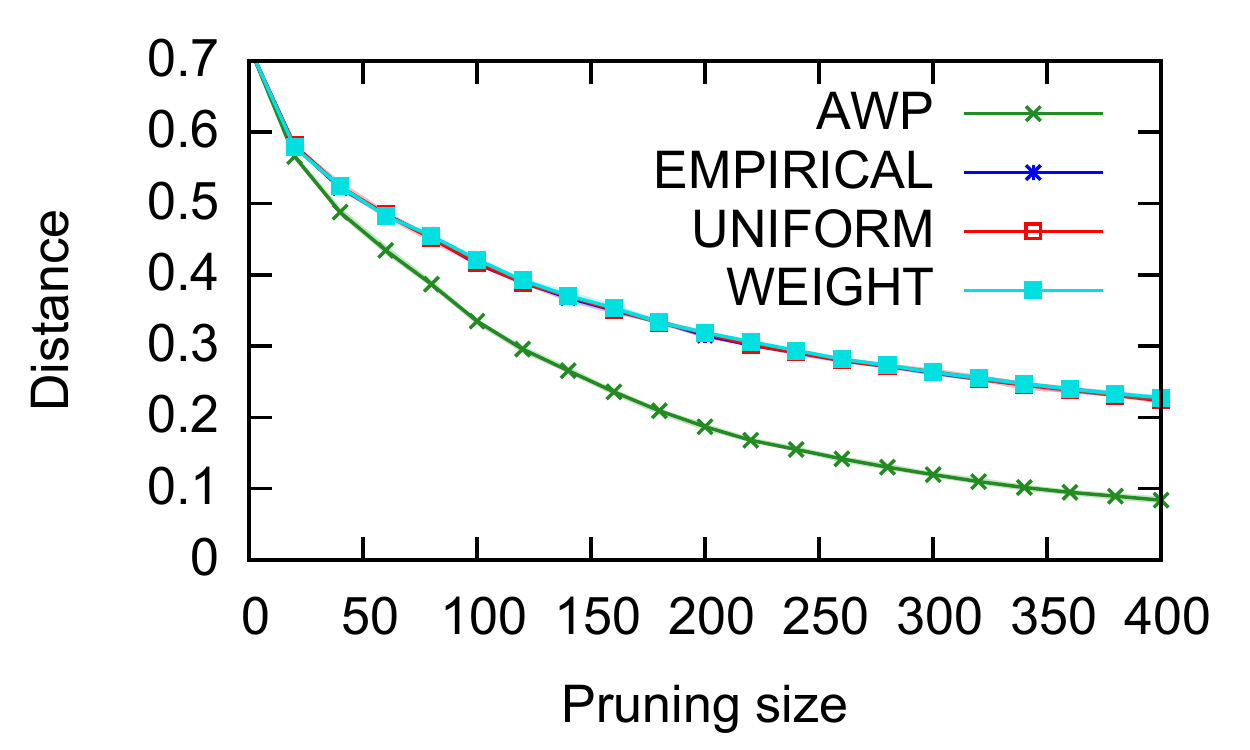}\\
  \includegraphics[width=0.5\textwidth]{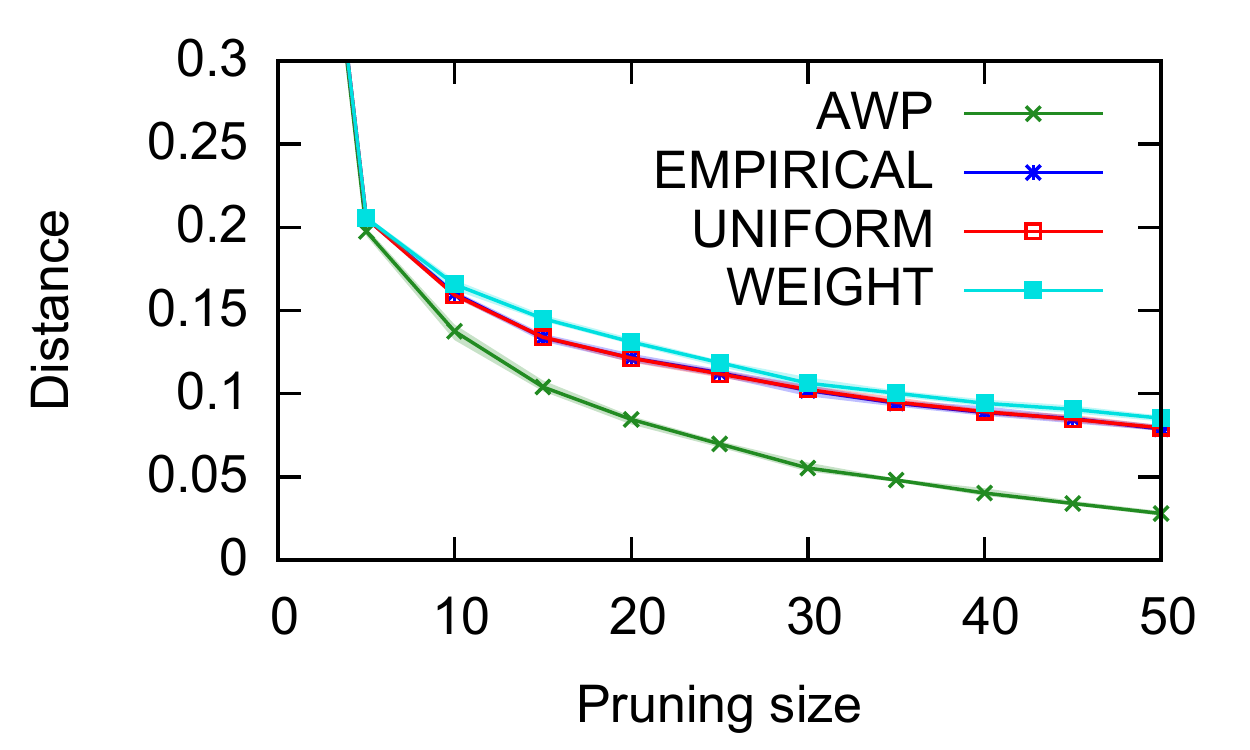}
  \includegraphics[width=0.5\textwidth]{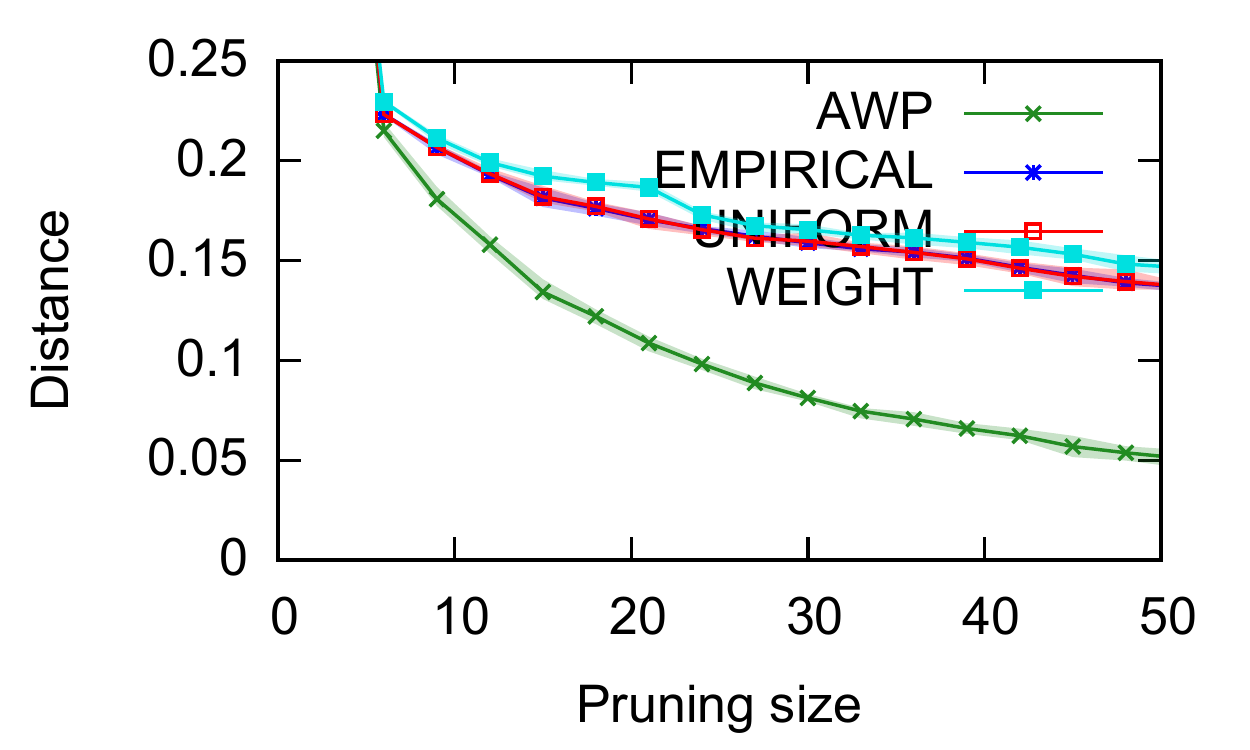}

	\caption{Experiments  with bins allocated by
          brightness. The weight of an example is $N$ times heavier than an
          example in the next bin. Left: $N=2$, right: $N=4$. Top: MNIST, Bottom: Caltech.}
        \label{fig:bright}
      \end{figure}

\begin{figure}[h]
  \includegraphics[width=0.5\textwidth]{Plots_new/Mnist_randomcat_exp2_v1.pdf}
  \includegraphics[width=0.5\textwidth]{Plots_new/Mnist_randomcat_exp4_v1.pdf}\\
  
  \includegraphics[width=0.5\textwidth]{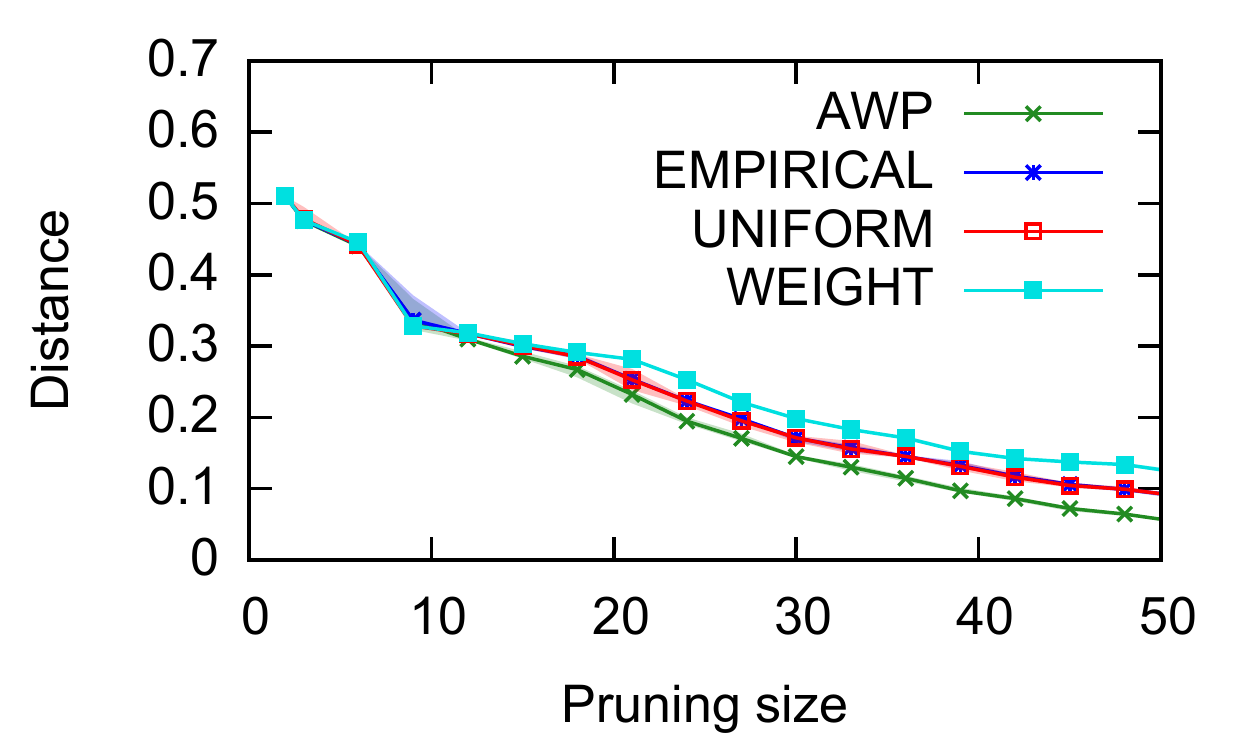}
  \includegraphics[width=0.5\textwidth]{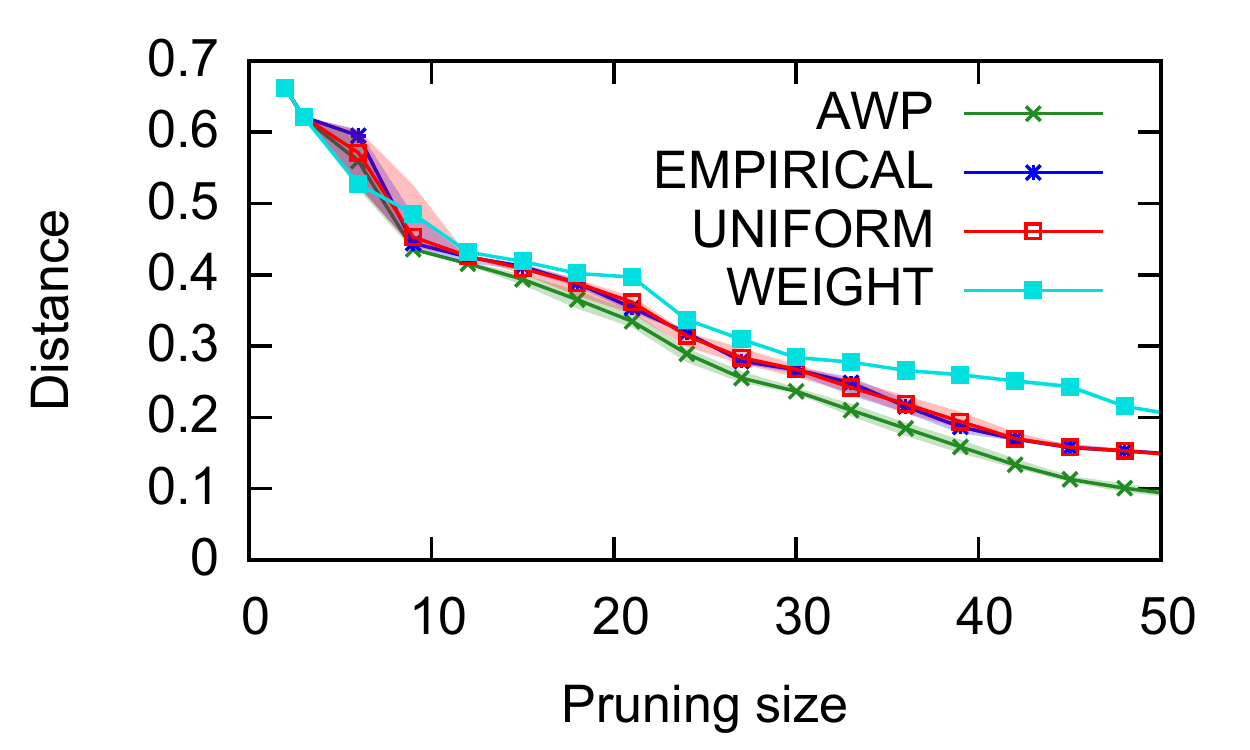}\\

  \includegraphics[width=0.5\textwidth]{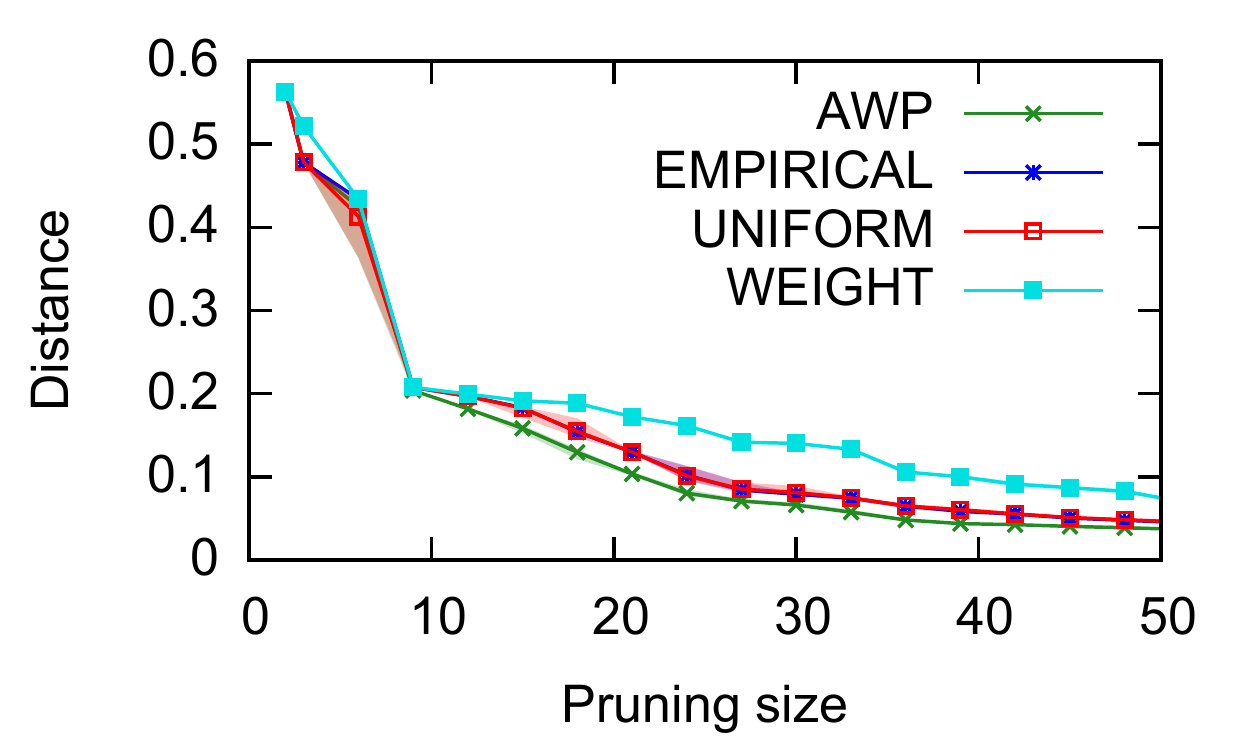}  \includegraphics[width=0.5\textwidth]{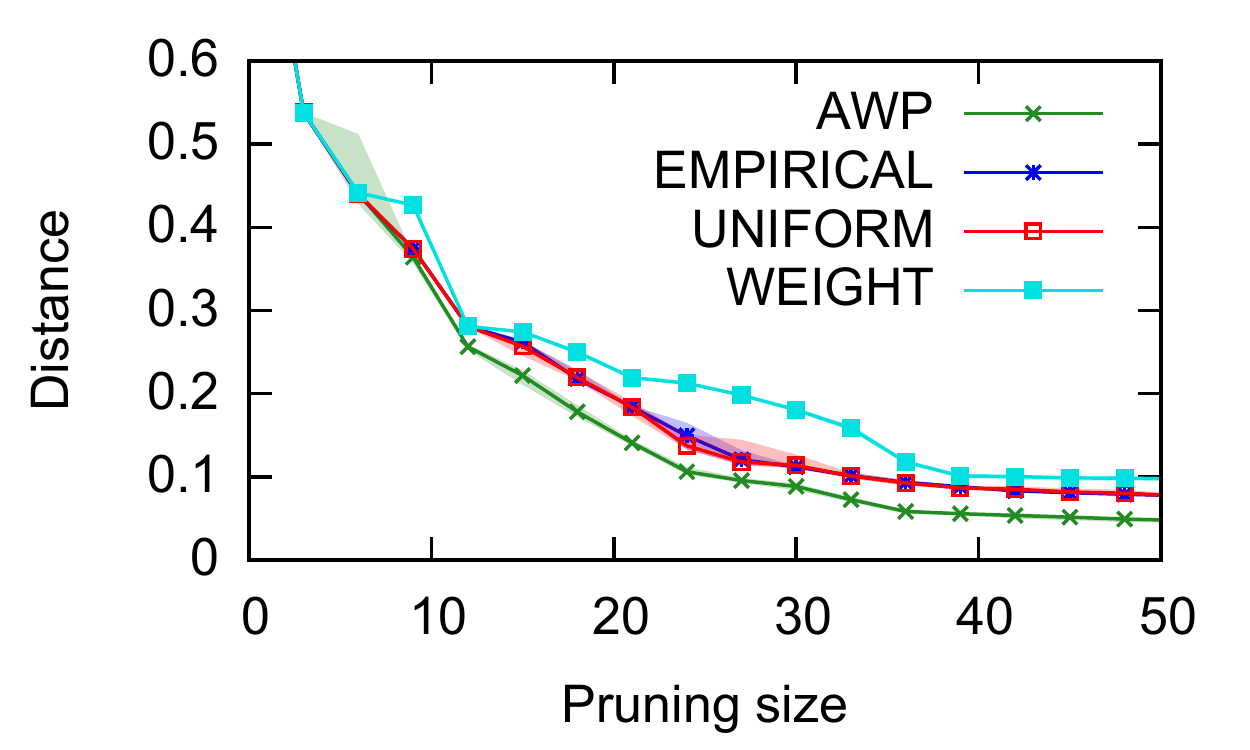}

	\caption{Experiments on the MNIST data set, with bins allocated by
          class. The weight of an example is $N$ times heavier than an
          example in the next bin. The three plots in each column show results for three
          random orders of classes. The left column shows $N=2$, and the right column shows $N=4$, for the same class orders. }
        \label{fig:mnistrand}
      \end{figure}

\begin{figure}[h]
  \includegraphics[width=0.5\textwidth]{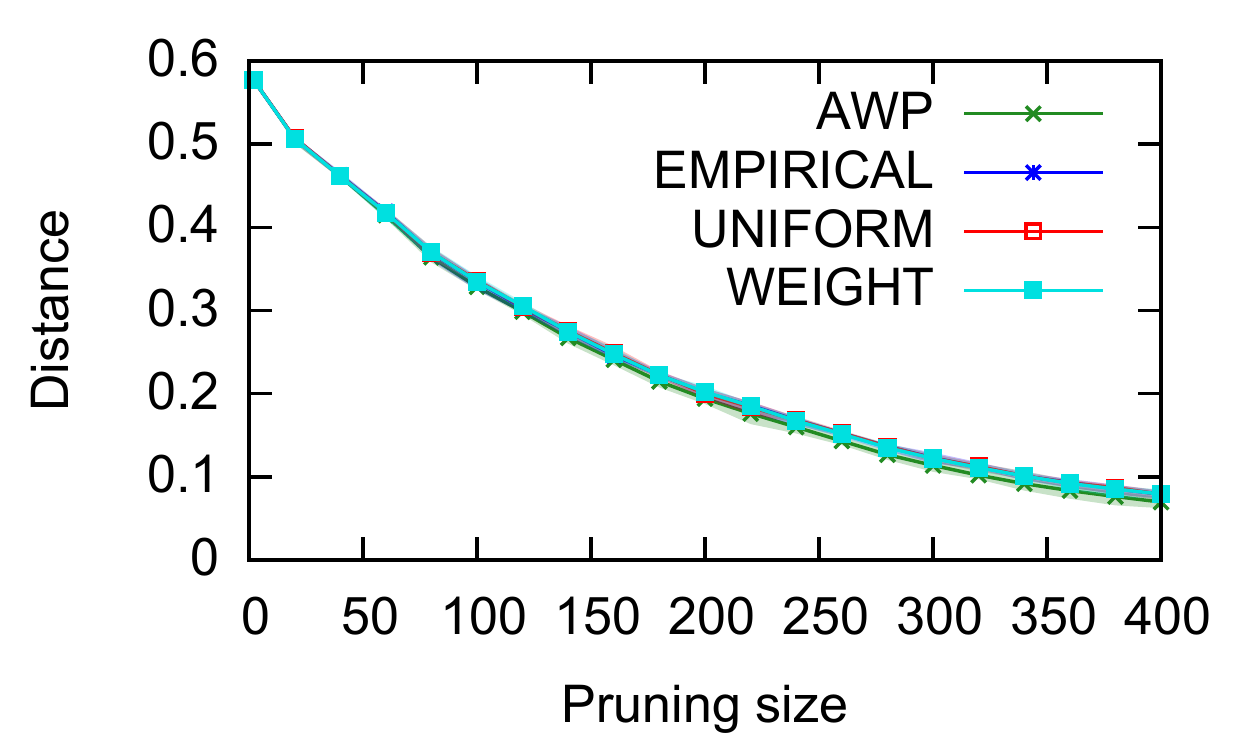}
  \includegraphics[width=0.5\textwidth]{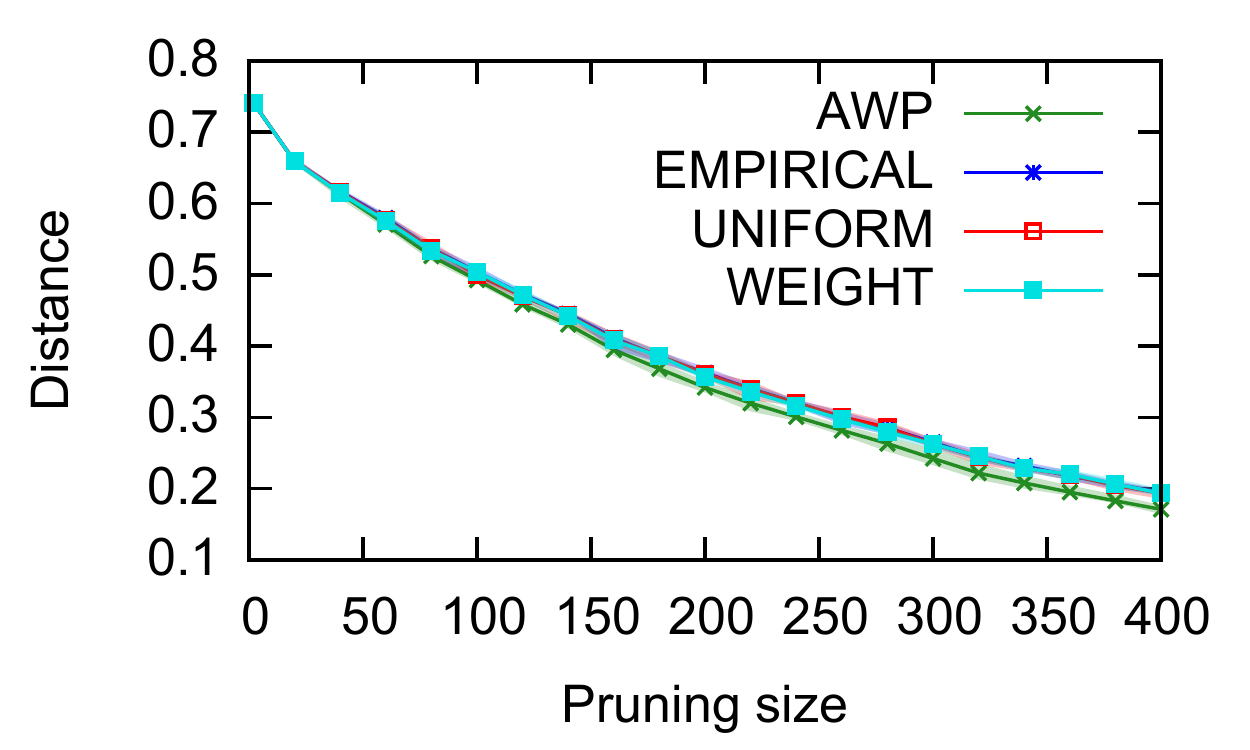}\\
  \includegraphics[width=0.5\textwidth]{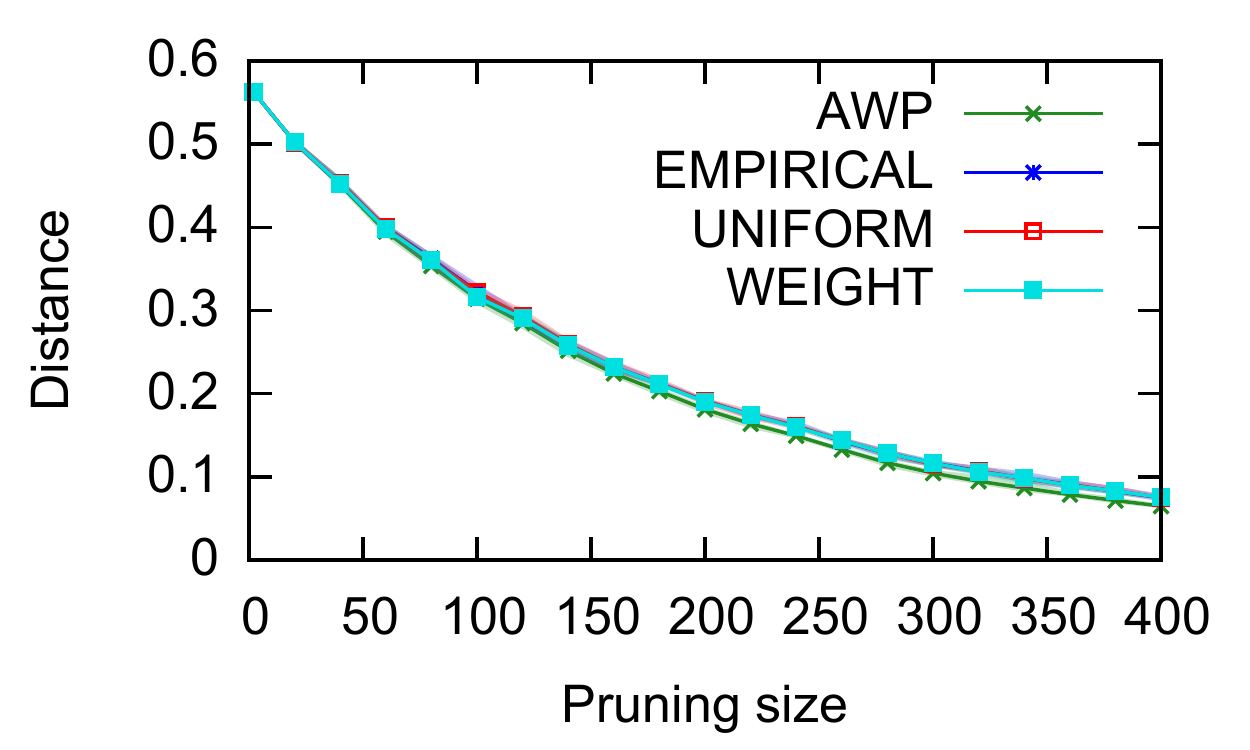}
  \includegraphics[width=0.5\textwidth]{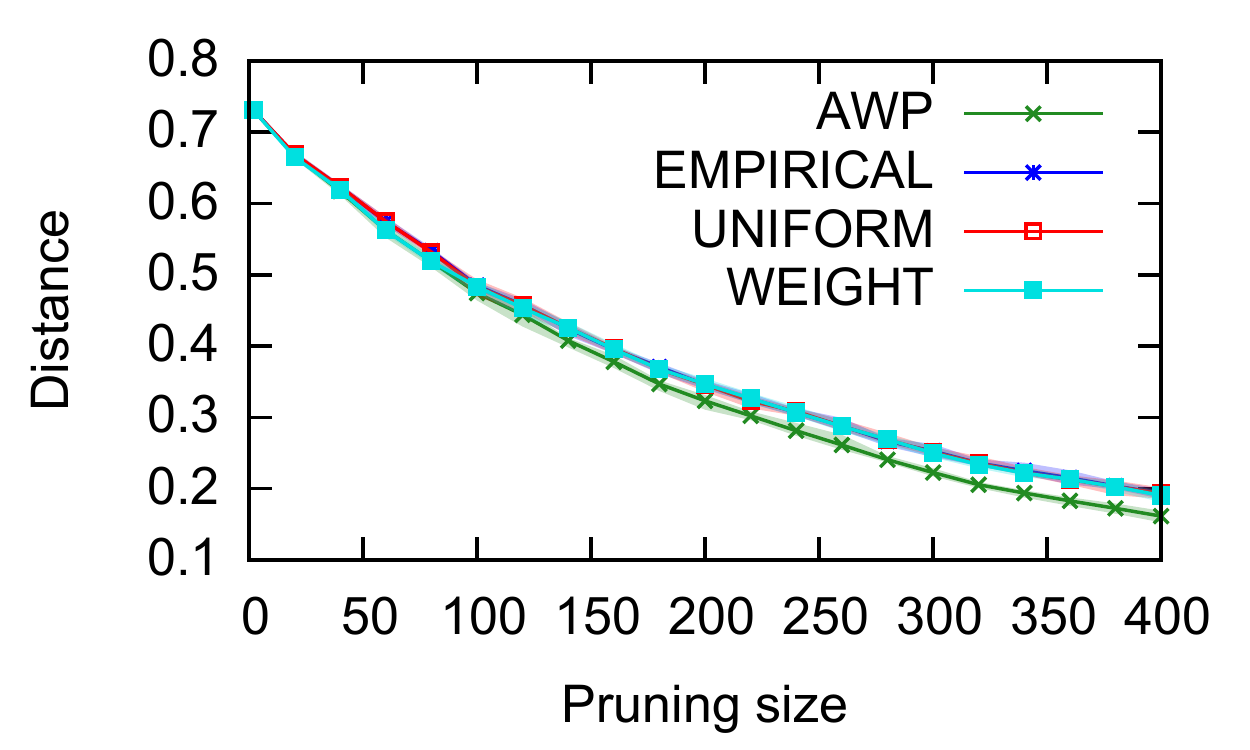}\\
  \includegraphics[width=0.5\textwidth]{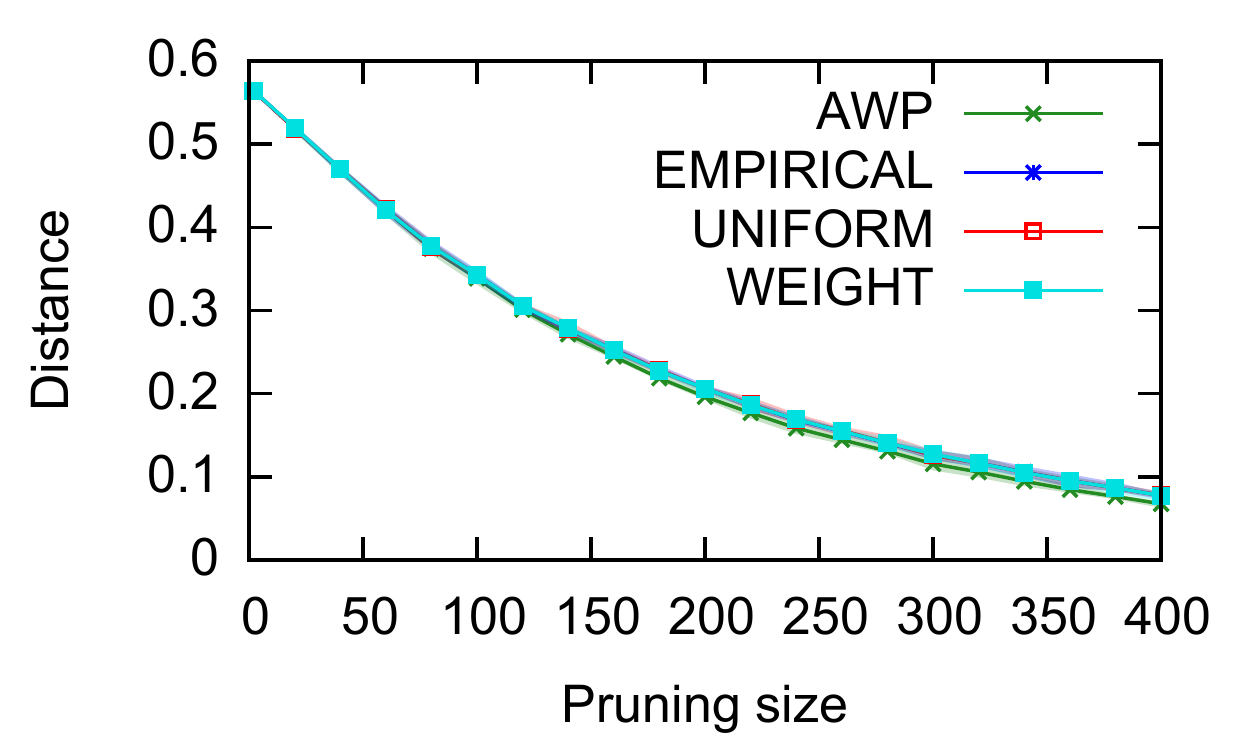}	
	\includegraphics[width=0.5\textwidth]{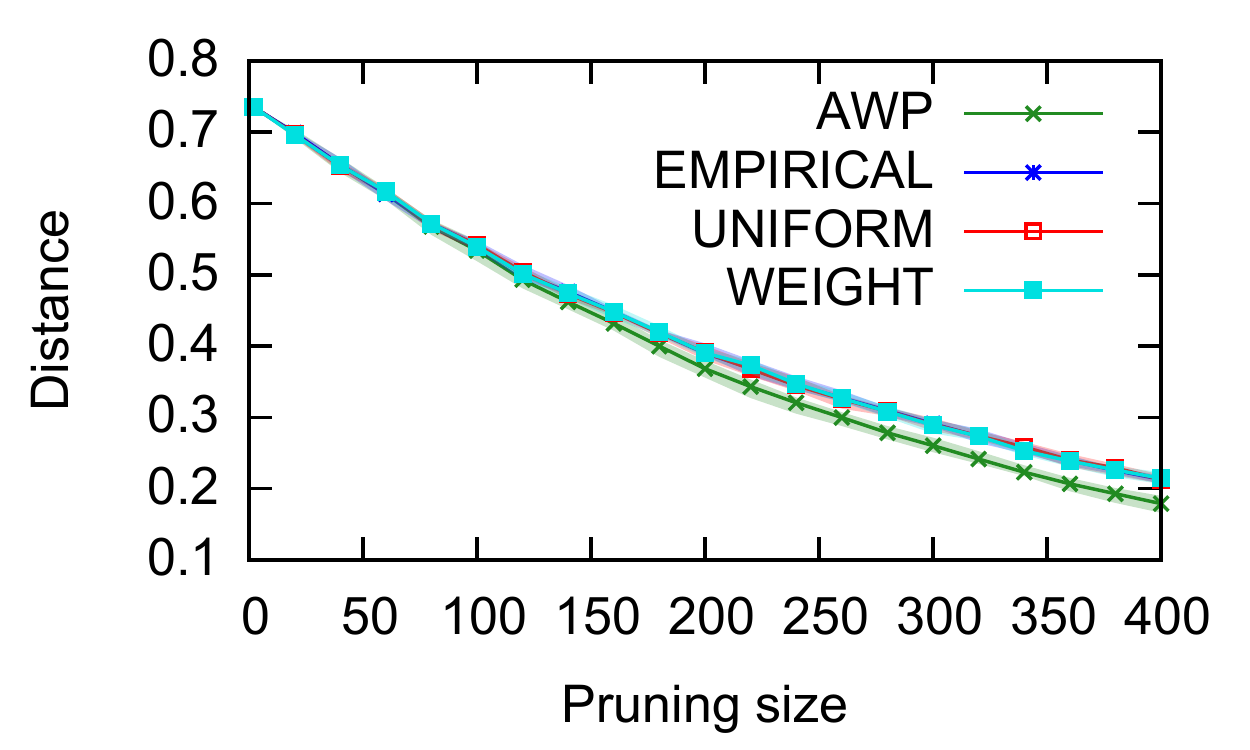}
	\caption{Experiments on the Caltech256 data set, with bins allocated by
          class. The weight of an example is $N$ times heavier than an
          example in the next bin. The three plots in each column show results for three
          random orders of classes. The left column shows $N=2$, and the right column shows $N=4$. }
        \label{fig:caltechrand}
      \end{figure}

\begin{figure}[t]
	\includegraphics[width=0.5\textwidth]{Plots_new/Caltech_1NN.pdf}
	\includegraphics[width=0.5\textwidth]{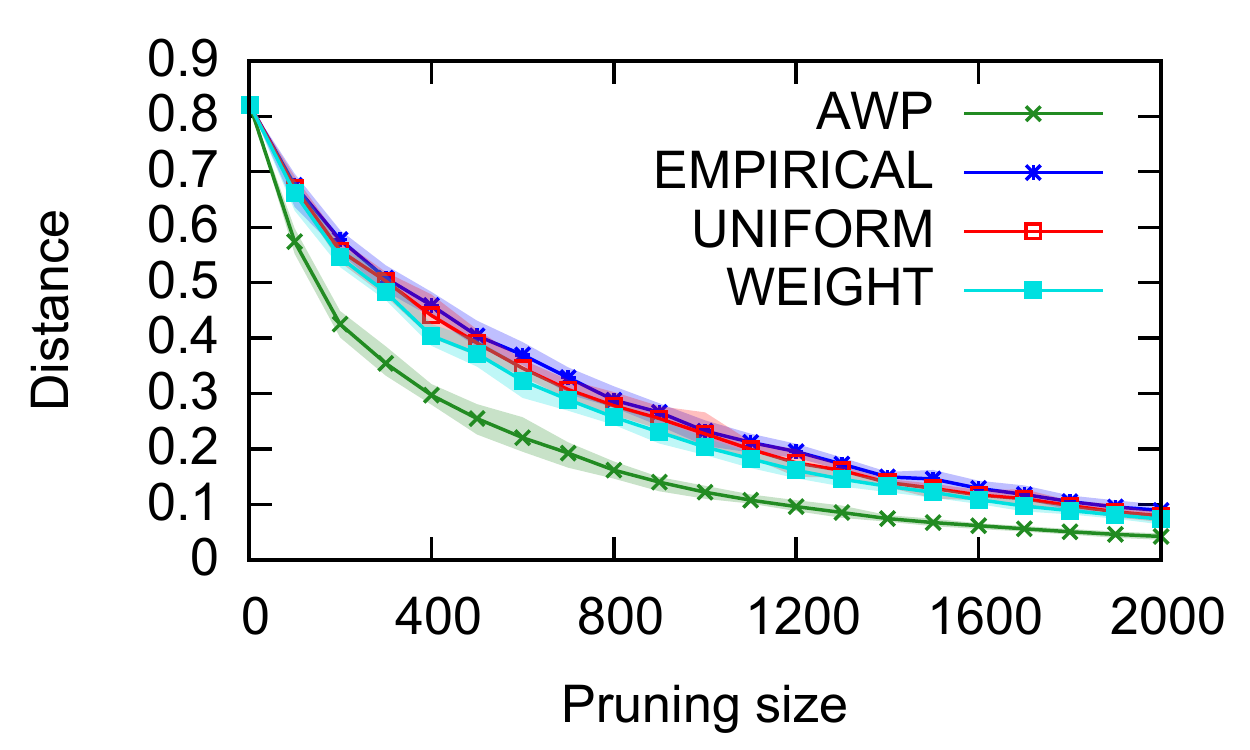}
	\caption{Experiments in which the input data set was Caltech256 and the target weights were calculated using other data sets. Left: the Office data set. Right: the Bing dataset.}
	\label{fig:nn}
\end{figure}

\begin{figure}[t]
  \begin{center}
        	\includegraphics[width=0.49\textwidth]{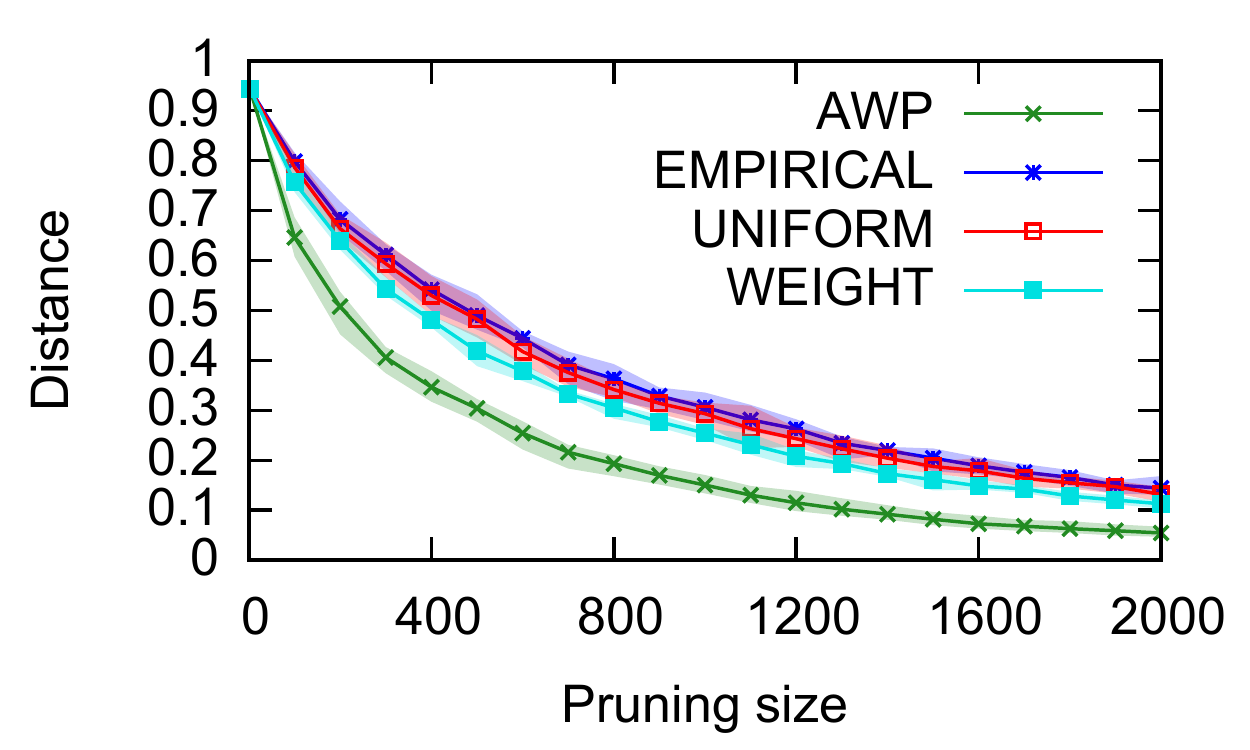}
                \includegraphics[width=0.49\textwidth]{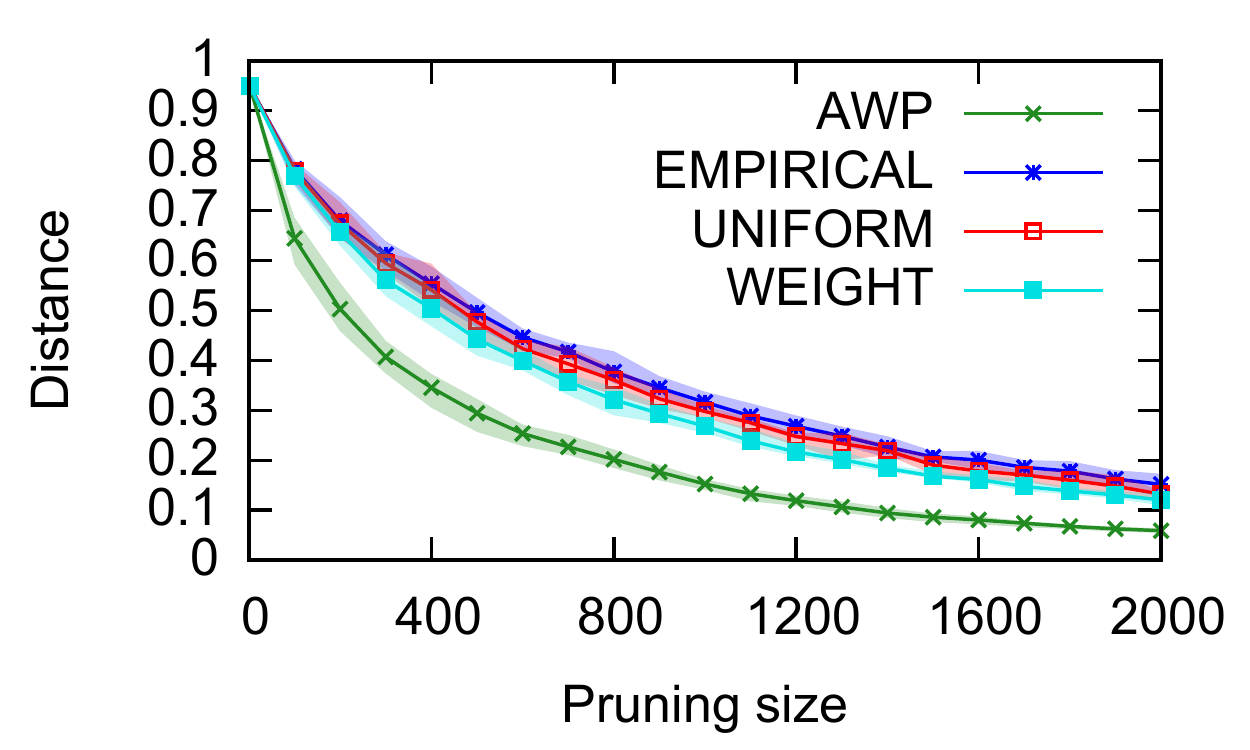}
                \includegraphics[width=0.49\textwidth]{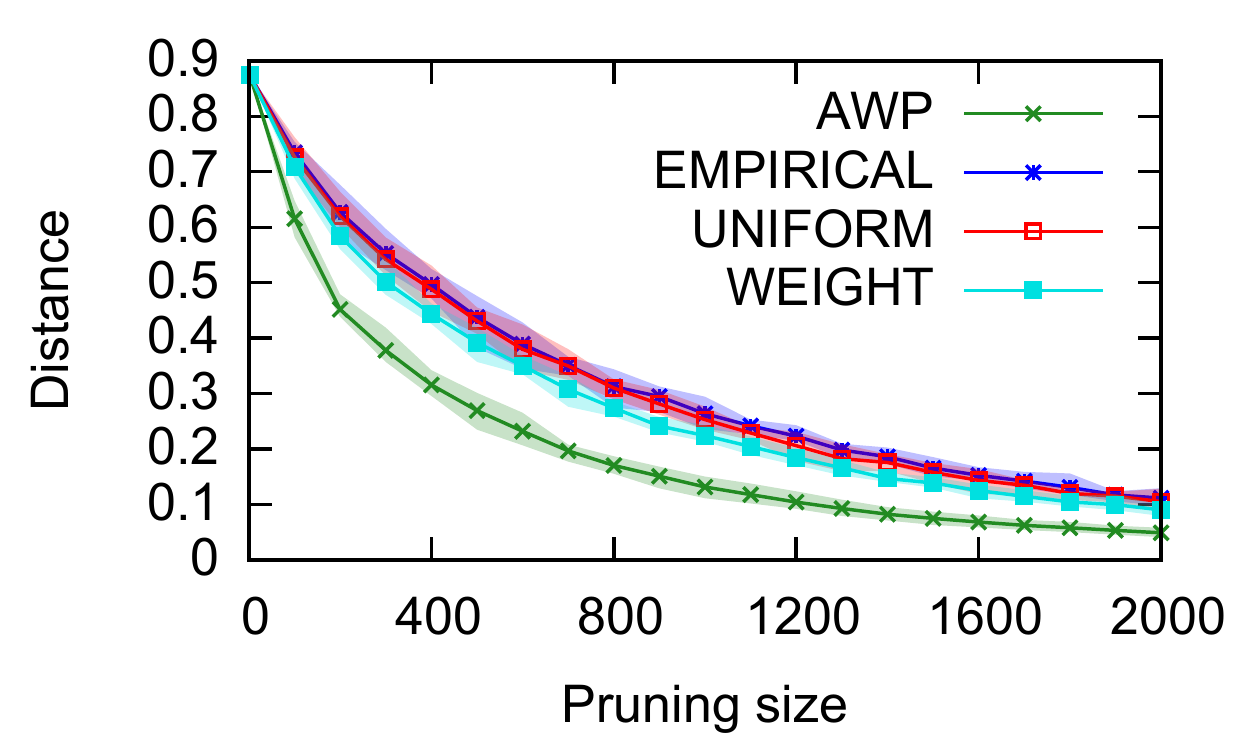}
              \end{center}
              \caption{Experiments in which the input data set was Caltech256 and the target weights were calculated using super classes from the Bing dataset. Top Left: The ``plants'' super class. Top Right: The ``insects'' super class. Bottom: The ``animals'' super classes. See \tabref{cat} for details on each super class.}
                \label{fig:BingSupercat}
\end{figure}

\begin{table}[b]
  \caption{The classes in each of the super-classes used in the reported experiments. The numbers in parentheses refer to the number of classes in the super-class.}
  \label{tab:cat}
  \begin{center}
    \resizebox{\textwidth}{!}{
  \begin{tabular}{ccc|cccc}
    Office (10) & Plants (10) & Insects (8) & \multicolumn{4}{c}{Animals (44)} \\
    \hline \\
    backpack & palm tree & butterfly & bat  & hummingbird & horse & horseshoe-crab  \\
    touring-bike & bonsai & centipede & bear  & owl & iguana & crab \\
    calculator & cactus & cockroach & camel & hawksbill & kangaroo & conch \\
    headphones & fern & grasshopper & chimp & ibis & llama & dolphin \\
    computer keyboard & hibiscus & house fly &  dog & cormorant & leopards & goldfish \\
    laptop & sun flower & praying-mantis & elephant  & duck & porcupine & killer-whale \\
    computer monitor & grapes & scorpion &  elk & goose & raccoon & mussels \\
    computer mouse & mushroom & spider &  frog & iris & skunk & octopus \\
    coffee mug & tomato &  & giraffe  & ostrich & snail & starfish \\
    video projector & water melon &  & gorilla & penguin & toad & snake  \\
     &  &  & greyhound & swan & zebra &  goat
  \end{tabular}
  }
\end{center}
\end{table}

\clearpage
\section{Average split quality in experiments}\label{app:qval}
As mentioned in \secref{exps}, in our experiments, the split quality $q$ was usually $1$ or very close to one. This shows that \algname\ can be successful even in cases not strictly covered by \thmref{main}. 
To gain additional insight on the empirical properties of the trees used in our experiments, we calculated also the average split quality of each tree, defined as the average over the set of values $\{\max(\dev_{v_R},\dev_{v_L}) / \dev_v \mid v \in T , \, \dev_v >0 \}$. The resulting values for all our experiments are reported in \tabref{averageqvalues}.

\begin{table}[h]
	\caption{The average split quality in each of the experiments. Left: experiments with target weight set by bins. In the ``classes'' experiments, the reported value is an average of the three tested random configurations. Right:  domain adaptation experiments with Caltech as the input data set and another data set as the target distribution.}
	\label{tab:averageqvalues}
	\begin{center}
          \begin{tabular}{cccc}
            Data set & Binning criterion & \multicolumn{2}{c}{Average split quality}\\

                     & &$N=2$ & $N=4$\\
            \toprule
			Adult & ``occupation'' & 0.574 & 0.544  \\
			Adult & ``relationship'' & 0.618 & 0.541  \\
			Adult & ``marital-status'' & 0.608 & 0.569  \\
			Adult & ``education-num'' & 0.587 & 0.547  \\
			MNIST & classes & 0.699 & 0.589  \\
			MNIST & brightness & 0.677 & 0.592  \\
			Caltech & classes & 0.645 & 0.652  \\
			Caltech & brightness & 0.677 & 0.641  \\
          \end{tabular}
          \hspace{2em}
		\begin{tabular}{cc}
			Target Data set & Average split quality \\
			\toprule 
			Office & 0.763  \\
			Bing & 0.756  \\
			Bing (plants) & 0.751  \\
			Bing (insects) & 0.755  \\
			Bing (animals) & 0.762  \\
		\end{tabular}
		
	\end{center}
\end{table}

\vfill
\end{document}